\documentclass{article}

\usepackage{microtype}
\usepackage{graphicx}
\usepackage{subcaption}%subfigure}
\usepackage{booktabs} % for professional tables

\usepackage{hyperref}

\usepackage[accepted]{icml2023}

\usepackage{amsmath}
\usepackage{amssymb}
\usepackage{mathtools}
\usepackage{amsthm}

\usepackage{algorithm,algorithmic}
\usepackage{caption}

\usepackage[capitalize,noabbrev]{cleveref}

\theoremstyle{plain}
\newtheorem{theorem}{Theorem}[section]
\newtheorem{proposition}[theorem]{Proposition}

\theoremstyle{definition}

\theoremstyle{remark}

\theoremstyle{plain}
\newtheorem{manualPropositioninner}{Proposition}
\newenvironment{manualProposition}[1]{%
	\renewcommand\themanualPropositioninner{#1}%
	\manualPropositioninner
}{\endmanualPropositioninner}

\usepackage[textsize=tiny]{todonotes}

\newcommand{\hide}[1]{}
\newcommand{\Exp}{\mathbb{E}}
\newcommand{\prob}{\mathrm{P}}

\newcommand{\argmin}[1]{\underset{#1}{\operatorname{argmin}}}

\newcommand{\R}{\mathbb{R}}

\newcommand{\w}{\mathbf{w}}
\newcommand{\vb}{\mathbf{v}}
\newcommand{\gb}{\mathbf{g}}

\newcommand{\thetab}{{\boldsymbol{\theta}}}
\newcommand{\distrib}{\mathcal{D}}

\newcommand{\lb}{\boldsymbol{l}}
\newcommand{\oneb}{\mathbf{1}}

\newcommand{\cond}{\mathcal{C}}
\newcommand{\MSBE}{MSBE}

\newcommand{\wb}{\mathbf{w}}
\newcommand{\mb}{\mathbf{m}}
\newcommand{\xb}{\mathbf{x}}
\newcommand{\yb}{\mathbf{y}}

\newcommand{\betab}{\boldsymbol{\beta}}
\newcommand{\nub}{\boldsymbol{\nu}}
\newcommand{\algcomment}[1]{$\qquad \triangleright$ {\tt\small #1}}
\newcommand{\defeq}{\overset{\text{\tiny def}}{=}}
\newcommand{\lamax}{\lambda_{\textrm{max}}}
\newcommand{\lamin}{\lambda_{\textrm{min}}}

\newcommand{\bindent}{%
	\begingroup
	\setlength{\itemindent}{12pt}
}
\newcommand{\bindentb}{%
	\begingroup
	\setlength{\itemindent}{12pt}
}
\newcommand\eindent{\endgroup}

\newif\ifbibtex
\bibtexfalse

\ifbibtex
\newcommand{\citer}[1]{}
\else
\newcommand{\citer}[1]{{\color{black}#1}}
\renewcommand{\cite}[1]{}
\fi

\newcommand{\todonext}[1]{}

\long\def\com#1{}%{{\color[rgb]{0.0,.65,.3}[#1]}}

\newcommand\redsout{\bgroup\markoverwith{\textcolor{red}{\rule[0.5ex]{2pt}{.5pt}}}\ULon}
\newcommand\bluesout{\bgroup\markoverwith{\textcolor{blue}{\rule[0.5ex]{2pt}{.5pt}}}\ULon}

\icmltitlerunning{Toward Efficient Gradient-Based Value Estimation}

\begin{document}

\twocolumn[
\icmltitle{Toward Efficient Gradient-Based Value Estimation
			}

\begin{icmlauthorlist}
	\icmlauthor{Arsalan Sharifnassab}{yyy} 
	%\,\, and \quad
	\icmlauthor{Richard Sutton}{yyy}
\end{icmlauthorlist}

\icmlaffiliation{yyy}{Authors are with the Department of Computing Science, University of Alberta, Canada}

\icmlcorrespondingauthor{Arsalan Sharifnassab}{sharifna@ualberta.ca}
\icmlkeywords{Reinforcement learning, value estimation, gradient-based methods}

\vskip 0.3in
]

\printAffiliationsAndNotice{}

\begin{abstract}
	Gradient-based methods for value estimation in reinforcement learning have favorable stability properties, but they are typically much slower than Temporal Difference (TD) learning methods. 
	We study the root causes of this slowness and show that Mean Square Bellman Error (MSBE) is an ill-conditioned loss function  in the sense that its Hessian has large condition-number. 
	To resolve the adverse effect of poor conditioning of MSBE on gradient based methods, we propose a low complexity batch-free proximal method that approximately follows the Gauss-Newton direction and is asymptotically robust to parameterization.
	Our main algorithm, called RANS,  is efficient in the sense that it is  significantly faster than the residual gradient methods while having almost the same computational complexity, and  is competitive with TD on the classic problems that we tested.
\end{abstract}

\section{Introduction} \label{sec:intro}
Value estimation is a core problem in reinforcement learning \citer{(Sutton \& Barto, 2018)}\cite{SuttB18}, and is a key ingredient  in several policy optimization methods, e.g., \citer{(Bhatnagar et al., 2009; Minh et al., 2015; Lillicrap et al., 2015)} \cite{Bhat09,mnih2015DQN,Lillicrap2015DDPG}.
The popular class of value estimation algorithms based on temporal difference learning via forward bootstrapping; including TD($\lambda$) \citer{(Sutton, 1988)}\cite{sutton1988TD}, Expected Sarsa \citer{(van Seijen et al., 2009)}\cite{vanSeijen2009ExpSarsa}, and Q-learning \citer{(Watkins \& Dayan, 1992)}\cite{watkins1992Qlearning} have found substantial empirical success when combined with proper policy optimization  \citer{(Minh et al., 2015; Minh et al., 2016; Lillicrap et al., 2015)}\cite{mnih2015DQN,mnih2016A3C,Lillicrap2015DDPG}.   
Nevertheless, these algorithms are not gradient-based optimization methods \citer{(Barnard, 1993)}\cite{Barn93} and their convergence cannot be guaranteed for general function approximation setting \citer{(Baird, 1995; Tsitsiklis \& Van Roy, 1997;  Brandfonbrener \& Bruna, 2019)}\cite{Bair95,tsitsiklis1996TDexample,BranB19}.
The stability problem of TD learning has inspired other classes of value estimation algorithms that involve optimizing a loss function through gradient updates. 
This  includes Residual Gradient (RG) algorithm for minimizing Mean Squared Bellman Error (MSBE) \citer{(Baird, 1995)}\cite{Bair95},  Gradient-TD  algorithms for minimizing projected Bellman error \citer{(Sutton et al., 2009; Maei et al., 2010; Maei, 2011; Hackman, 2012)}\cite{SuttMPBSSW09,maei2010toward,maei2011thesis,hackman2012msThesis}, and their extensions for optimizing a dual formulation of BE2 \citer{(Liu et al., 2015; Macua et al., 2014; Dai et al., 2017)}\cite{LiuGMP20,macua2014distributed,DaiHPBS17}.
These  algorithms enjoy the general robustness and convergence properties of Stochastic Gradient Descent (SGD), but are known to be slower than TD  in  tabular and linear function approximation settings \citer{(Baird, 1995; Schoknecht \& Merke, 2003; Gordon, 1999; Ghiassian \& Sutton, 2021)} \cite{Bair95,schoknecht2003td,gordon1999thesis,Ghiassian2021empirical4roorms}.

In this paper, we investigate the root causes of the slowness problem of gradient-based value estimation by taking a  deeper look into the landscape of MSBE, and propose linear complexity methods to alleviate these problems.
We provide theoretical results showing that MSBE is  an ill-conditioned loss function in the senses that the condition-number of its Hessian matrix is typically very large.  
This explains slowness of gradient-based value estimation methods, because gradient descent in general is slow in minimizing ill-conditioned loss functions.
In contrast, algorithms like Newton and Gauss-Newton methods are invariant to conditioning of the loss. 
Unfortunately a direct implementation of these methods requires matrix inversion, which is computationally costly even if computed incrementally. 
We propose a linear complexity incremental algorithm, called Residual Approximate Gauss-Newton (RAN), that incorporates a trace to approximate the Gauss-Newton direction and then updates the weights along that trace. We show that RAN can be equivalently formulated as a batch-free proximal algorithm. 
A weakness of RAN is that it requires double sampling \citer{(Baird, 1995)}\cite{Bair95}, which limits its use in stochastic environments. 
We propose a double-sampling-free extension of RAN by following  similar ideas that underlie GTD-type methods.
The resulting algorithms significantly outperform RG and GTD2, being  orders of magnitudes faster on the simple classic environments that we tested, while having almost similar computational complexity to RG and GTD2.

We then turn our focus to a second cause of slowness of gradient-based value estimation:
under function approximation, 
sample gradients of MSBE involve large outliers that carry important information, resulting in large variance of stochastic updates.  
Outliers of this type often appear in every episode (usually at pre-terminal transitions), and are specific to gradient-based value estimation methods (i.e. such outliers do not appear in TD learning).
We propose a general technique called outlier-splitting, which results in no information loss as opposed to the standard clipping methods. Our main value estimation algorithm, called RAN with outlier-Splitting (RANS), has linear computational complexity and has only one effective hyper-parameter (and some other hyper-parameters that can be set to their default values), thanks to its adaptive step-size mechanism.
Our empirical results on  a few classic control environments with neural network function approximation show significant improvement over
 RG, and achieving competitive performance to TD. 

\section{Background}\label{sec:background} 
We consider a discounted Markov Decision Process (MDP) defined by the tuple $(\mathcal{S},\mathcal{A}, \mathcal{R}, p, \gamma)$, where $\mathcal{S}$ is a finite set of states, $\mathcal{A}$ is a finite set of actions, $\mathcal{R}$ is a set of rewards,  $p:\mathcal{S}\times\mathcal{A}\times\mathcal{S}\times\mathcal{R}\to[0,1]$ is the environment dynamics determining the probability of the next state and immediate reward given a current state and action pair, and $\gamma\in[0,1]$ is a discount factor. %, and $\distrib_0:\mathcal{S}\to[0,1]$ is an initial distribution of states. 
We fix a stationary policy $\pi:\mathcal{S}\times\mathcal{A}\to[0,1]$, and let $p_\pi(s',a',r|s,a)=p(S_{t+1}=s',R_{t+1}=r|S_t=s,A_t=a)\pi(A_{t+1}=a'|S_{t+1}=s')$. % denote the probability of transiting to state $s'$ and receiving reward $r$ then tking action $a'$, when we start from action $a$ at state $s$, under policy $\pi$.
We consider an episodic and online setting where a data stream $(S_1,A_1,R_1),(S_2,A_2,R_2),\ldots$ is generated according to the policy $\pi$.
The action-value function $q_\pi:\mathcal{S}\times\mathcal{A}\to\R$, at each state $s$ and action $a$, is the expected discounted sum of rewards obtained by   starting from state $s$ and action $a$ and following policy $\pi$. 
We then define the value function $v_\pi:\mathcal{S}\to\R$ as  $v_\pi(s)=\Exp_{a\sim\pi(\cdot|s)}[q_\pi(s,a)]$.

In value estimation, we aim to obtain an estimate  of the true action-values $q_\pi$, usually through a  function $q_\w:\mathcal{S}\times\mathcal{A}\to\R$ parameterized by a $d$-dimensional weight vector $\w$. 
Corresponding to $q_\w$ is  a Bellman residual at each state and action pair $(s,a)$, defined as
\begin{equation*}
\delta_\w(s,a) \defeq \Exp_{s',a',r\sim p_\pi(\cdot, \cdot,\cdot|s,a)} \big[r+\gamma q_\w(s',a')-q_\w(s,a)\big].
\end{equation*}
\hide{
	\subsection{Gradient-based value estimation}
}
According to Bellman equations \citer{(Sutton \& Barto, 2018; Bertsekas \& Tsitsiklis, 1996)}\cite{SuttB18,BertT96},  $q_\wb=q_\pi$  if and only if $\delta_\wb(s,a)=0$ for all $(s,a)\in\mathcal{S}\times\mathcal{A}$.
In this view, $MSBE(\wb)$, defined below, serves as a proxy for the quality of  estimates $\wb$:
\begin{equation}\label{eq:MSBE}
MSBE_{\distrib}(\wb) \,\defeq\, \Exp_{(s,a) \sim \distrib}\big[\delta_\wb(s,a)^2\big],
\end{equation}
where $\distrib$ is some distribution over states and action pairs. 
When the distribution is online, we drop the subscript $\distrib$ and write $MSBE(\cdot)$. For simplicity of notation, we also write $\Exp_{s,a}[\cdot]$ to denote the expectation with respect to state and action pairs sampled from the online distribution.
In the same vein, we consider a parameterized estimate $v_\w:\mathcal{S}\to\R$ of value function $v_\pi$, and let 
\begin{equation}\label{eq:MSBE v}
MSBE^V_{\distrib}(\wb) \,\defeq\, \Exp_{s \sim \distrib}\big[\delta_\wb(s)^2\big],
\end{equation}
where $\delta_\wb(s)=\Exp_{a\sim \pi(\cdot|s)}[\delta_\wb(s,a)]$, and $\distrib$ is some distribution over states.

Gradient-based value estimation methods use gradient-based optimization algorithms to minimize MSBE or other related objectives such as  MSPBE \citer{(Sutton et al., 2009)}\cite{SuttMPBSSW09}. 
The first and simplest  method in this category is the RG algorithm \citer{(Baird, 1995)}\cite{Bair95}.
In this algorithm, to obtain an unbiased sample estimate of $\nabla_\wb(\delta_\wb(S_t,A_t)^2)$,  we require independent samples $(S_{t+1},A_{t+1},R_t)$ and $(S_{t+1}',A_{t+1}', R_t')$ from $p_\pi(\cdot, \cdot,\cdot|S_t,A_t)$. 
For simplicity of notation, at time $t$, we let
\begin{align}
\delta_t &\defeq R_t+\gamma q_{\wb}(S_{t+1},A_{t+1})-q_{\wb}(S_{t},A_{t}),      \label{eq:delta t}\\
\delta_t' &\defeq R_t'+\gamma q_{\wb}(S_{t+1}',A_{t+1}')-q_{\wb}(S_{t},A_{t}),  \label{eq:delta' t}
\end{align}
where $\wb_t=\wb$.
The RG update is then 
\begin{equation}\label{eq:RG with DS}
\wb\leftarrow \wb -\alpha \delta_t'\nabla_\wb\delta_t.
\end{equation}
%where $\delta_t$ and $\delta_t'$ are defined in \eqref{eq:delta t} and \eqref{eq:delta' t}, respectively.
The requirement for two independent sample transitions at time $t$ is called double sampling \citer{(Sutton \& Barto, 2018)}\cite{SuttB18}. 

In stochastic environments, double sampling is possible only if we have a correct model of the world. In other words, MSBE minimizer is \emph{not learnable} if an exact model of the underlying stochastic environment is not available, which is the case in real-world applications \citer{(Sutton \& Barto, 2018)}\cite{SuttB18}.
A general technique to circumvent double sampling is using Fenchel duality  to obtain an equivalent saddle point formulation of MSBE \citer{(Dai et al., 2017, Du et al., 2017)}\cite{DaiHPBS17, du2017stochastic} as
\begin{equation}
\min_\wb\max_{\hat{\delta}(\cdot,\cdot)} \Exp_{s,a}\left[\delta_\wb(s,a) \,\hat\delta(s,a) - \frac12 \hat\delta(s,a)^2  \right],
\end{equation}
where $\hat\delta(s,a)$ is an auxiliary variable that serves as a proxy of $\delta_\wb(s,a)$. 
In practice, one can consider a parametric approximation $\hat{\delta}_{\thetab}(\cdot,\cdot)$ of $\hat\delta(\cdot,\cdot)$,
and perform gradient updates on the resulting minimax problem $\min_\wb\max_{\thetab} \Exp_{s,a}\left[\delta_\wb(s,a) \,\hat\delta_\thetab(s,a) - \frac12 \hat\delta_\thetab(s,a)^2  \right]$:
\begin{equation}\label{eq:gtd2}
\begin{split}
\wb&\leftarrow \wb -\alpha \hat\delta_\thetab(S_t,A_t)\nabla_\wb\delta_t,\\
\thetab&\leftarrow  \thetab+ \eta\big(\delta_t - \hat\delta_\thetab(S_t,A_t)\big)\nabla_\thetab \hat\delta_\thetab(S_t,A_t),
\end{split}
\end{equation}
\citer{(Sutton et al., 2009; Liu et al., 2020; Dai et al., 2017)}\cite{SuttMPBSSW09,LiuGMP20,DaiHPBS17}. 
Intuitively, this is similar to the RG algorithm in \eqref{eq:RG with DS} except for using the parametric approximation $\hat\delta_\thetab(S_t,A_t)$ instead of  $\delta_t'$, and updating
$\hat\delta_\thetab(s,a)$ by SGD on $\Exp_{s,a}\big[(\hat\delta_\thetab(s,a)-\delta_\wb(s,a))^2\big]$.
The GTD2 algorithm \citer{(Sutton et al., 2009)}\cite{SuttMPBSSW09} is a special case of \eqref{eq:gtd2} in which $q_\wb$ and $\hat\delta_\thetab$ are linear approximations of the form $q_\wb(s,a)=\boldsymbol{\phi}_{s,a}^T\wb$ and $\hat\delta_\thetab(s,a)=\boldsymbol{\phi}_{s,a}^T\thetab$, for feature vectors $\boldsymbol{\phi}_{s,a}$ \citer{(Liu et al., 2020; Dai et al., 2017)}\cite{LiuGMP20,DaiHPBS17}.

\section{MSBE loss is ill-conditioned} \label{sec:cond}
The condition-number of  a symmetric square matrix, $H$, is defined as the ratio of its largest to smallest singular values, 
$\max_{\xb:\|\xb\|=1} |\xb^TH\xb|/\min_{\yb:\|\yb\|=1}  |\yb^TH\yb|$.  
For a quadratic function $f(\xb)=\xb^TH\xb$, 
we define the condition-number, $\cond(f)$, of $f$ as the condition-number of its Hessian matrix $H$. 
Intuitively, level sets (or contours) of a convex quadratic function have an elliptical shape, and the condition-number $\cond(f)$ equals the squared ratio between the largest and the smallest diameters of each of these ellipsoids (see Fig.~\ref{fig:cond traj}). 
We say that $f$ is ill-conditioned if $\cond(f)$ is very large. 
Then, the level sets of an ill-conditioned quadratic function have the shape of ellipsoids that are thin.  
It is known that the 
convergence rate of the gradient descent on a quadratic loss $f$ scales with $\cond(f)$ \citer{(Polyak, 1964)}\cite{Poly64}, which can be very slow for  ill-conditioned loss functions.

In this section, we consider linear function approximation.
In this case, $MSBE^V_\distrib(\cdot)$ defined in \eqref{eq:MSBE v} is a convex quadratic  function. We denote the condition-number of $MSBE^V_\distrib(\cdot)$ under uniform distribution $\distrib$ by $\cond$.  We let $l$ be the average episode length, defined as the expected time until termination when starting from a state, uniformly averaged over all states. We also let $h\defeq\Exp_{s\sim\textrm{unif}}\big[\prob(S_{t+1}=s|S_t=s)\big]$ be the the average self-loop probability. 
Note that $h$ is typically much smaller than $1$.

\begin{theorem}\label{th:cond}
In the tabular case, the following statements hold for any discount factor $\gamma\in [0,1]$:
\begin{itemize}
\item[a)]
For any MDP and under any policy, we have
\begin{equation} \label{eq:C1}
	\cond \ge \frac{(1-\gamma h)^2}{4} \, \min\left( \frac1{(1-\gamma)^2},\, l^2\right)
\end{equation}
where $l$ is the average episode length and $h$ is the average self-loop probability. %(i.e., self-loop probability of a state, uniformly averaged over all states).
\item[b)]
For any $n\ge1$, there exists an $n$-state MDP and a policy for which $\cond\ge\gamma^4 n^2/(1-\gamma)^2$.
\end{itemize}
\end{theorem}
The proof is given in Appendix~\ref{app:proof cond val}.
A similar result also holds for the condition-number of $MSBE$ defined in \eqref{eq:MSBE} (see Proposition~\ref{prop:cond q} in Appendix~\ref{app:proof cond action-val}).
Theorem~\ref{th:cond} shows that  MSBE is typically ill-conditioned in the tabular case. This explains the slow convergence of gradient-based methods for minimizing MSBE. 
As an example, the bound in \eqref{eq:C1} implies that for  $\gamma=.99$ and for any MDP and policy pair with average episode length at least $100$ and average self-loop probability no larger than $0.1$, we have $\cond > 2000$.
Moreover, Theorem~\ref{th:cond}~(b) implies  that for  $\gamma=.99$, there is a $100$-state Markov chain for which $\cond>96,000,000$.

Lower bounds similar to Theorem~\ref{th:cond} are not possible for  non-tabular linear function approximation. 
%For more general non-tabular linear function approximation case, there exists no universal lower bound on $\cond$. 
This is because different feature representations can improve  or worsen the condition-number. To see why, note that in the linear function approximation case and under uniform state distribution, $\MSBE^V_{\textrm{unif}}(\wb) =\wb^T \Phi^T(I-\gamma P)^T(I-\gamma P)\Phi \wb$, where $\Phi$ is  an $n\times d$ matrix, each row of which is a feature vector of a state; and $P$ is the transition matrix. For the specific choice $\Phi = (I-\gamma P)^{-1}$ we obtain $\cond=1$, while for the case that $\Phi$ is not full-rank, we  have $\cond=\infty$.

\begin{figure}[t!]
	\centering % the figures are generated by 1-Experiment_2D_trajectories
	\includegraphics[width=.65\linewidth]{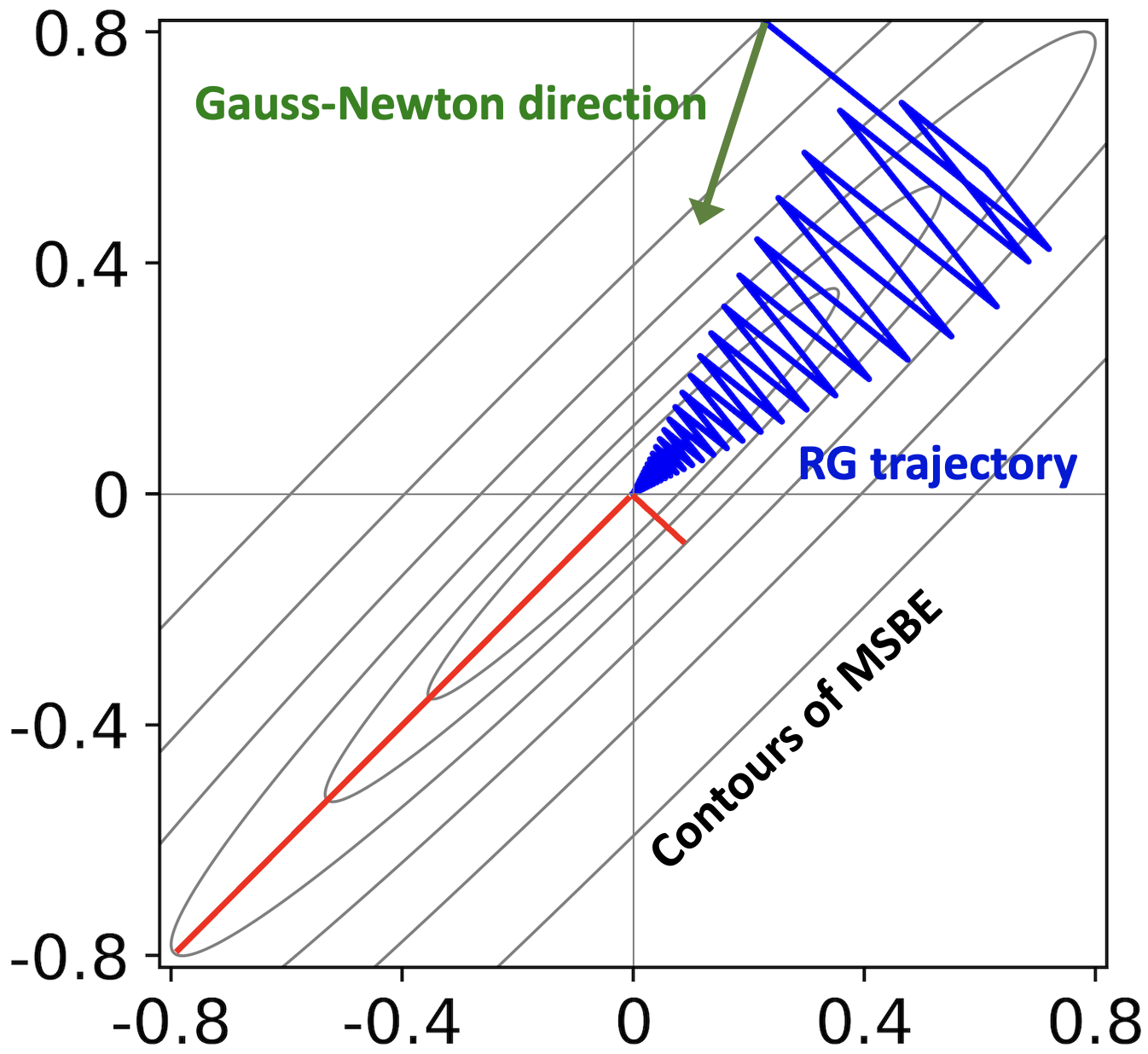}
	\caption{Level sets of MSBE (gray curves)  in a 2-state loop environment with $p(s_0\to s_1)=p(s_1\to s_0)=1$ for $\gamma=0.8$. Here, condition number of MSBE is 81, and is equal to the squared ratio between the diameters (red) of each ellipsoid. The solution trajectory of  RG  (blue) for $\alpha=.9$ and the Gauss-Newton direction (green) are also depicted. In this environment, $\cond=O(1/(1-\gamma)^2)$ (see Theorem~\ref{th:cond}), which rapidly grows for larger $\gamma$.
	}\label{fig:cond traj}
	%\end{center}
\end{figure}

In general, since underparameterized function approximation reduces  parameters dimension, it usually improves condition-number.
\hide{We establish this claim in a specific %(and rather restrictive) 
regime for random feature matrices by employing tools from random matrix theory in Appendix~\ref{app:cond large random matrix}[To Do, or delete].}
Fig.~\ref{fig:boyan} illustrates dependence of $\cond$ on the number of features, $d$, in an extended Boyan chain environment \citer{(Boyan, 2002)}\cite{Boya02} with 200 states and with random binary features (see Appendix~\ref{app:empirical details boyan} for details). We observe that smaller $d$ results in better condition-number; but this comes at the cost of larger value-error at MSBE minimum (the red curve in Fig.~\ref{fig:boyan}), where by value-error  we mean $\Exp_{s\sim\textrm{unif}}\big[\Exp_{a\sim\pi(\cdot|s)}[\|q_\wb(s,a)-q_\pi(s,a)\|^2]\big]$. 
See Appendix~\ref{app:cond large random matrix} for more experiments on  condition number under linear function approximation.

\begin{figure}[t!]
	\centering % the figures are generated by code_condition_number_large_random_matrix.py in my short codes
	\includegraphics[width=.8\linewidth]{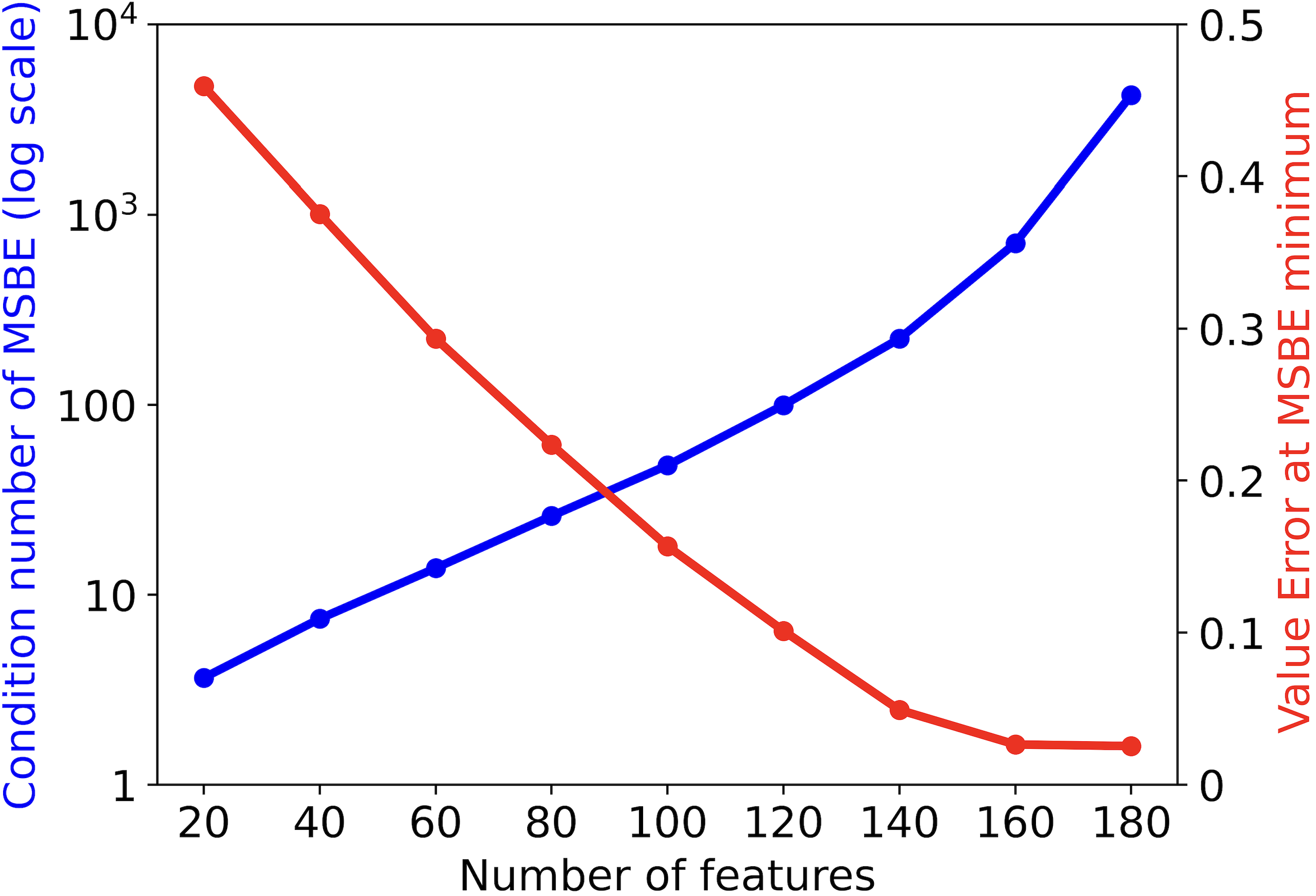}
	\caption{Condition-number of MSBE (blue) and value-error at MSBE minimizer (red) versus number of features, in a 200-state extended Boyan chain with random binary features.}\label{fig:boyan}
	%\end{center}
\end{figure}

\hide{
\begin{figure}[t!]
	\centering % the figures are generated by code_condition_number_large_random_matrix.py in my short codes
	\begin{subfigure}[t]{0.8\linewidth}
		\centering
		\includegraphics[width=1\linewidth]{BE2_conditioning_BoyanChain_random_binary_features.png}
		\caption{}
	\end{subfigure}%
\\
	\begin{subfigure}[t]{.8\linewidth}
		\centering
		\includegraphics[width=1\linewidth]{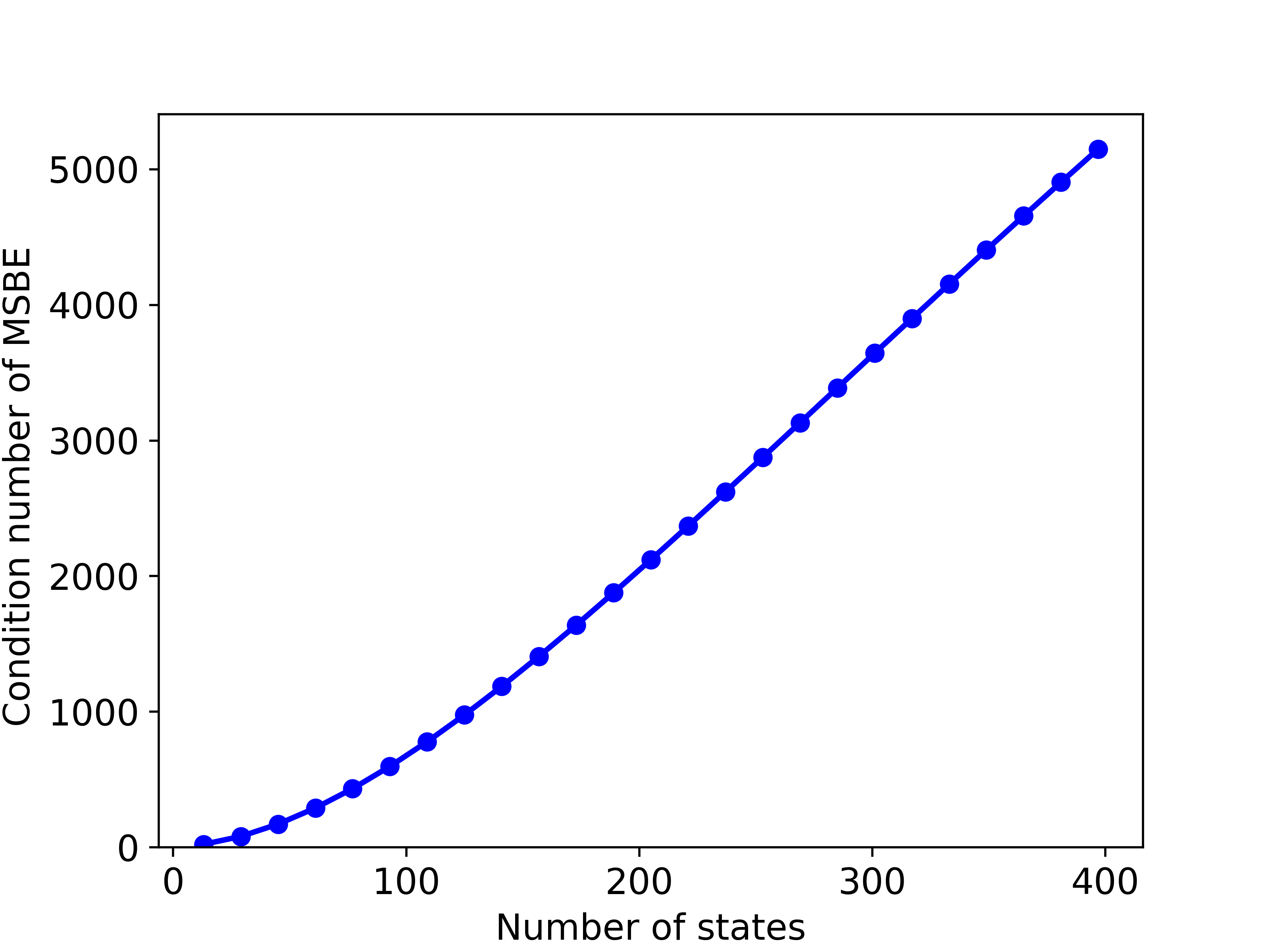}
		\caption{}
	\end{subfigure}
		\caption{Condition-number of MSBE in extended Boyan chain under linear function approximation. (a) shows condition-number (blue) and value-error at MSBE minimizer (red) versus number of features, in a 200-state extended Boyan chain with random binary features. (b) depicts the dependence of condition-number on the  number of states, in extended Boyan chains with standard features ($n=4d-3$). }\label{fig:boyan full}
	%\end{center}
\end{figure}
}

\section{A review of  the Gauss-Newton method}
Consider an expected loss function of the form  $F(\wb)=\Exp_f[f^2(\wb)]$, and the associated Hessian matrix $H_F=\Exp[\nabla f \nabla f^T]+\Exp[f\,H_f]$, where $H_f$ denotes Hessian of sample function $f$. 
The first term on the right hand side, $\Exp[\nabla f \nabla f^T]$, is called the Gauss-Newton matrix and is denoted by $G$. The Gauss-Newton algorithm then updates $\wb$ as $\wb\leftarrow \wb-\alpha G^{-1 }\nabla F(\wb)$. 
In the special case that functions $f$ are linear, we have $H_f=0$ and thereby $H_F=G$. 
In this case, Gauss-Newton and  Newton methods become equivalent.
However, the Gauss-Newton algorithm has two advantages in the   non-linear case. 
Firstly, $G^{-1} \nabla F(\wb)$ is always a descent direction, as opposed to the Newton updates that may climb uphill and  converge to local maxima or saddle points \citer{(Nesterov \& Polyak, 2006)}\cite{nesterov2006CRN}. 
Secondly, $G$ can be computed in terms of gradients while $H$ entails second order derivatives which are not as easily accessible in certain settings \citer{(Nocedal \& Wright, 1999)}\cite{nocedal1999numerical}.

As opposed to gradient descent which is prohibitively slow in ill-conditioned problems, Newton and Gauss-Newton methods are invariant to  conditioning of the loss function. 
Some recent works proposed using Gauss-Newton method for value estimation \citer{(Gottwald et al., 2021)}\cite{GottGSD21}. However, these algorithms require matrix inversion, which is computationally costly even if 
computed incrementally  via Sherman-Morrison formula with quadratic complexity \citer{(Sherman \& Morrison, 1950)}\cite{Sherman1950adjustment}. 
In the next section, we propose an incremental low complexity approach that approximately follows the Gauss-Newton direction.

\section{Our first algorithm: RAN} \label{sec:alg1}
The Gauss-Newton direction for minimizing MSBE is $\mb_{GN}(\wb)=G_\wb^{-1} \nabla MSBE(\wb)/2$, where $G_\wb=\Exp_{s,a}\big[\nabla\delta_\wb(s,a)\,\nabla\delta_\wb(s,a)^T\big]$ is the Gauss-Newton matrix and 
$\nabla MSBE(\wb) =2\Exp_{s,a}\big[\delta_\wb(s,a)\,\nabla\delta_\wb(s,a)\big]$. 
Then, $\mb_{GN}$ is the minimizer of the following quadratic function:
\begin{equation}\label{eq:quad loss}
	L(\mb) \defeq \frac12\, \Exp_{s,a}\left[   \big( \delta_\wb(s,a) -\nabla \delta_\wb(s,a)^T \mb \big)^2    \right]  .
\end{equation}
This is because for any $\mb$,
\begin{equation*}
	\begin{split}
	\nabla_\mb L(\mb) &=  \Exp_{s,a}\left[   \big( \nabla \delta_\wb(s,a)^T \mb - \delta_\wb(s,a)\big)\nabla \delta_\wb(s,a)  \right]\\
	&=G_\wb \mb - \nabla MSBE(\wb)/2,
	\end{split}
\end{equation*}
and therefore $\nabla_\mb L(\mb_{GN})=\mathbf{0}$.
We follow a two time scale approach  \citer{(Yao et al., 2009, Bhatnagar et al., 2009; Dabney \& Thomas, 2014)}\cite{yao2009, Bhat09,DabnT14} 
to incrementally find an approximate minimizer $\mb$ of $L$  and update $\wb$ along that direction.
More concretely, given a $\beta>0$ and $\lambda\in[0,1]$, at time $t$, we update $\mb$ along an unbiased sample gradient of $\beta L(\mb)  +(1-\lambda)\|\mb\|^2$, 
\begin{equation} \label{eq:2TS delta'_T}
\begin{split}
\mb&\leftarrow  \lambda  \mb+ \beta\big(\delta_t'-\mb^T\nabla\delta_t'\big)\nabla \delta_t,
\end{split}
\end{equation}
where $\delta_t$ and $\delta'_t$ are defined in \eqref{eq:delta t} and \eqref{eq:delta' t}, and $(1-\lambda)\|\mb\|^2$ is a Levenberg–Marquardt regularizer \citer{(Marquardt, 1963)}\cite{Marq63}\footnote{
	We have empirically observed that $\lambda=1$ often leads to slow convergence, because it causes large inertia in $\mb$, and therefore  large oscillations in $\wb$. The  best performance is achieved for $\lambda\in(0.99,0.9999)$. }.
We then update $\wb$ along $\mb$, i.e, $\wb\leftarrow \wb-\alpha \mb$.

Algorithm~\ref{alg:ran} gives the pseudo code of the RAN algorithm. For better stability and faster convergence, the update of $\mb$ in RAN is  of the form
\begin{equation} \label{eq:2TS}
\begin{split}
\mb&\leftarrow  \lambda  \mb+ \beta\big(\delta_t'-\mb^T\nabla\delta_t\big)\nabla \delta_t.%\\
\end{split}
\end{equation}
which is the same as \eqref{eq:2TS delta'_T} except for using $\nabla\delta_t$ instead of $\nabla\delta_t'$.
The  updates in \eqref{eq:2TS} have lower variance compared to \eqref{eq:2TS delta'_T}, and additionally $ \nabla \delta_t \nabla\delta_t^T$ in \eqref{eq:2TS} is positive semi-definite, as opposed to the finite sample estimate of the Gauss-Newton matrix $\frac1\tau \sum_{t=1}^{\tau} \nabla \delta_t \nabla\delta_t'^T$ in \eqref{eq:2TS delta'_T}. 
For fixed $\wb$, the update in \eqref{eq:2TS}   is in expectation along 
\begin{equation} \label{eq:exp ran dir}
-\big(\Exp_t[\nabla\delta_t \nabla\delta_t^T]+(1-\lambda)I\big)^{-1}\nabla MSBE(\wb).
\end{equation}

In Appendix~\ref{app: proximal view of RAN}, we provide further intuition for RAN, by presenting a derivation of Algorithm~\ref{alg:ran} as a proximal method with momentum for minimizing MSBE. In this view,  $\mb$ serves as a momentum of MSBE gradients, to which we add a correction term equal to the gradient of a penalty function that aims to regularize the change in  $\delta_\wb(s,a)$ for all state action pairs $(s,a)$.

\begin{algorithm}[tb]
	\caption{\quad RAN}
	\label{alg:ran}
	\begin{algorithmic}
		\STATE {\bfseries Parameters:} step-sizes $\alpha,\beta$, and decay parameter $\lambda$
		\STATE {\bfseries Initialize:} $\mb=\mathbf{0}$ and $\wb$
		\FOR{$t=1,2,\ldots$}
		\STATE consider $\delta_t$ and $\delta'_t$ defined in \eqref{eq:delta t} and \eqref{eq:delta' t}, respectively
		\STATE $\mb \leftarrow  \lambda  \mb+ \beta \,\delta'_t\,\nabla \delta_t$
		\STATE $\mb \leftarrow   \mb-\beta (\mb^T\nabla\delta_t) \nabla \delta_t$
		\STATE $\wb\leftarrow \wb-\alpha\mb$
		\ENDFOR
	\end{algorithmic}
\end{algorithm}

Convergence of the RAN algorithm can be shown in two-time-scale regime where $\alpha_t,\beta_t\to0$, with $\alpha$ diminishing faster than $\beta$ (i.e., $\alpha_t/\beta_t\to0$)\footnote{
	Note that the two-time-scale view  is only for the purpose of convergence analysis, and in practice we consider fixed or adaptive step-sizes whose ratio needs not go to zero.}.
Convergence of such two-time-scale algorithms is well-studied \citer{(Kushner \& Yin, 2003; Konda \& Tsitsiklis, 1999; Bhatnagar et al., 2009)}\cite{KushY03, KondT99, Bhat09}, 
under some smoothness and irreducibly conditions.
In Appendix~\ref{app:convergence proof}, we discuss different conditions for convergence of Algorithm~\ref{alg:ran} in the two-time-scale regime.
Moreover, in this regime, RAN is robust to reparameterization:
\begin{proposition}[Informal]\label{prop:invariance}
	For $\lambda=1$ and asymptotically small step-sizes $\alpha\to0$ and $\alpha/\beta\to0$, the trajectory  of $\wb$ in the RAN algorithm is invariant to any differentiable and bijective non-linear transformation on parameterization. 
\end{proposition} 
The formal version of Proposition~\ref{prop:invariance} and its proof are given in Appendix~\ref{app:invariance}.

We evaluated the performance of RAN in a simple benchmark environment. Consider an environment with $n$ states and one action, in which each state $i=1,2\ldots,n$ transits to state $\min(i+1,n)$ with probability $1-\epsilon$, and transits to a terminal state with probability $\epsilon$, for some $\epsilon\in[0,1)$. 
This is a generalization of the Hallway environment \citer{(Baird, 1995)}\cite{Bair95}, and is known to be a challenging task for the RG algorithm \citer{(Baird, 1995)}\cite{Bair95}. 
We tested Algorithm~\ref{alg:ran} in this environment with $n=50$, $\epsilon=0.01$, and $\gamma=0.99$  in the tabular setting (see Appendix~\ref{app:exp Baird6} for the details of this experiment).
The learning curves are depicted in Fig.~\ref{fig:Baird50}.
We observe that, in this experiment, Algorithm~\ref{alg:ran}  is about 30 times faster than RG, and reaches a convergence rate close to TD(0) \citer{(Sutton, 1988; Sutton \& Barto, 2018)}\cite{sutton1988TD,SuttB18}.

\begin{figure}[t!]
	\centering % the figures are generated in folder 3- Experimant Baird6_with_2TS_3
	\includegraphics[width=.8\linewidth]{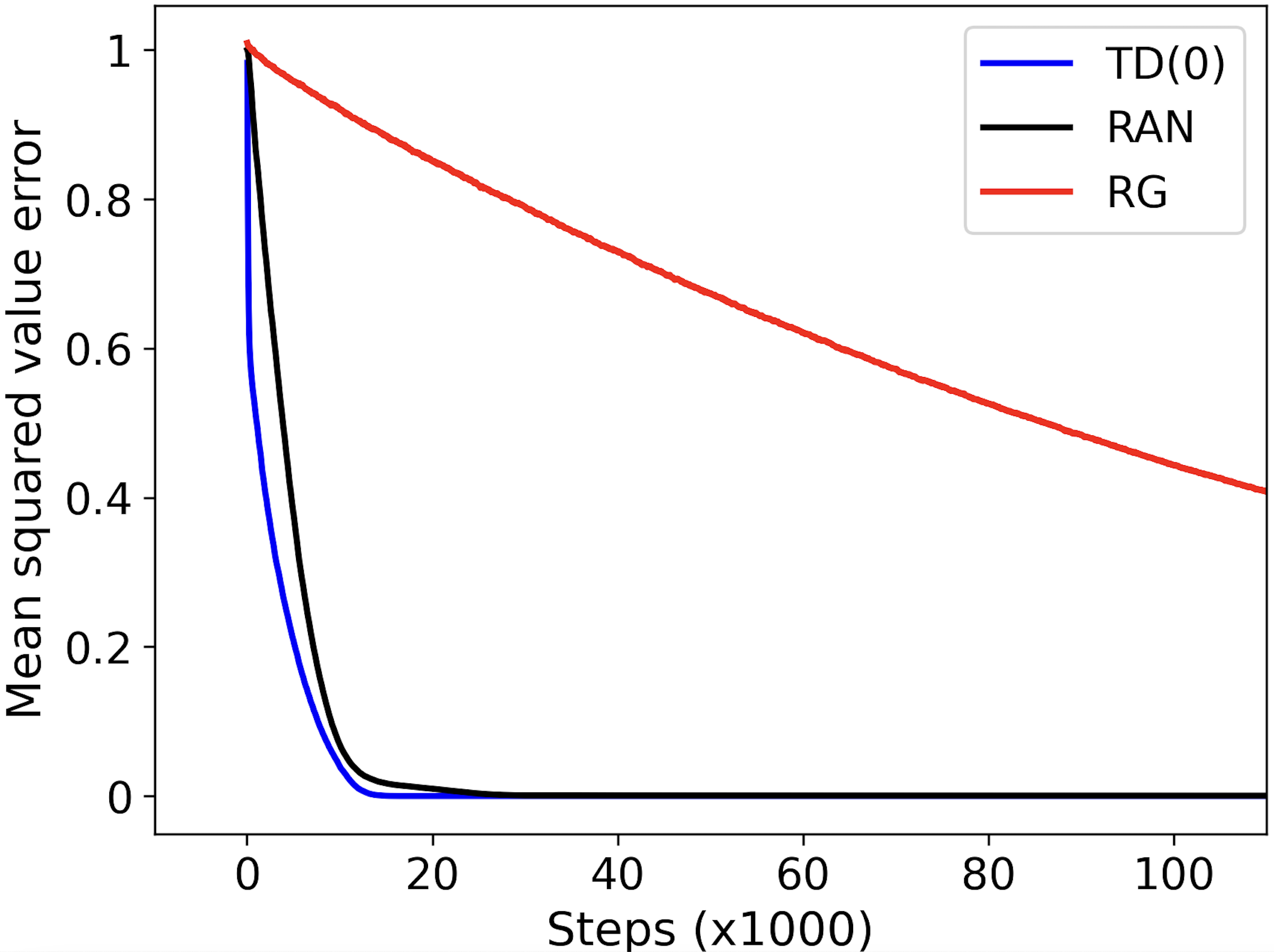}
	\caption{The Hallway experiment  discussed in Section~\ref{sec:alg1}.}\label{fig:Baird50}
	%\end{center}
\end{figure}

\section{ Double-sampling-free RAN algorithm}\label{sec:GTD}
In Algorithm~\ref{alg:ran}, we require double sampling to compute $\delta'_t$. In this section, we propose a Double-Sampling-Free version of RAN, called DSF-RAN. Double sampling is easily doable in deterministic environments \citer{(Saleh \& Jiang, 2019; Zhang et al., 2020)}\cite{saleh2019deterministic,ZhanBW2020}, in which case $\delta'_t$ can be computed using an independent sample $A'_{t+1}$ from the policy. However, for double sampling in stochastic environments, we require a model to get an independent sample $S'_{t+1}$ of the next state, which is typically possible only in simulated environments. 

To resolve the double sampling issue of RAN in stochastic environments,  
we use the  technique  discussed in Section~\ref{sec:background}, which was also used  in the GTD2 algorithm. 
More specifically, instead of $\delta'_t$ in Algorithm~\ref{alg:ran}, we use a parametric approximation $\hat\delta_\thetab(S_t,A_t)$ of $\delta_\wb(S_t,A_t)$, parameterized by $\thetab$. Similar to  GTD2 (see \eqref{eq:gtd2}), we then learn $\thetab$ through SGD on $\Exp_{s,a}\big[(\hat\delta_\thetab(s,a)-\delta_\wb(s,a))^2\big]$.
Pseudo code of DSF-RAN is  given in Algorithm~\ref{alg:dsf-rang}.

\begin{algorithm}[tb]
	\caption{\quad  DSF-RAN}
	\label{alg:dsf-rang}
	\begin{algorithmic}
		\STATE {\bfseries Parameters:} step-sizes $\alpha,\beta,\eta$, and decay parameter $\lambda$
		\STATE {\bfseries Initialize:} $\mb=\mathbf{0}$, $\wb$, $\thetab$.
		\FOR{$t=1,2,\ldots$}
		\STATE  $ \nabla \delta_t=\gamma\nabla_\wb q_\wb(S_{t+1},A_{t+1})-\nabla_\wb q_\wb(S_t,A_t)$
		\STATE $\mb \leftarrow  \lambda  \mb+ \beta \,\hat\delta_\thetab(S_t,A_t)\, \nabla \delta_t$
		\STATE $\mb \leftarrow   \mb-\beta (\mb^T\nabla\delta_t) \nabla \delta_t$
		\STATE $\wb\leftarrow \wb-\alpha\mb$
		\STATE $\thetab\leftarrow  \thetab+ \eta\big(\delta_t - \hat\delta_\thetab(S_t,A_t)\big)\nabla_\thetab \hat\delta_\thetab(S_t,A_t)$
		\ENDFOR
	\end{algorithmic}
\end{algorithm}

 We tested RAN and DNS-RAN algorithms on  Baird's Star environment \citer{(Baird, 1995)}\cite{Bair95}, that is  a Markov chain with six states, each represented by seven features  (see Appendix~\ref{app:exp Baird star} for details of this experiment).
The results are illustrated in Fig.~\ref{fig:Baird star}.
We observe that in this environment, RAN and DNS-RAN converge about 200 times faster than RG and GTD2 algorithms, respectively. 
It is well-known that off-policy TD(0) is unstable in this environment  \citer{(Baird, 1995)}\cite{Bair95}.
 
 \begin{figure}[t!]
 	\centering % the figures are generated in folder 2-Experiment Barid_star off-policy_with_2TS_2
 	\includegraphics[width=.8\linewidth]{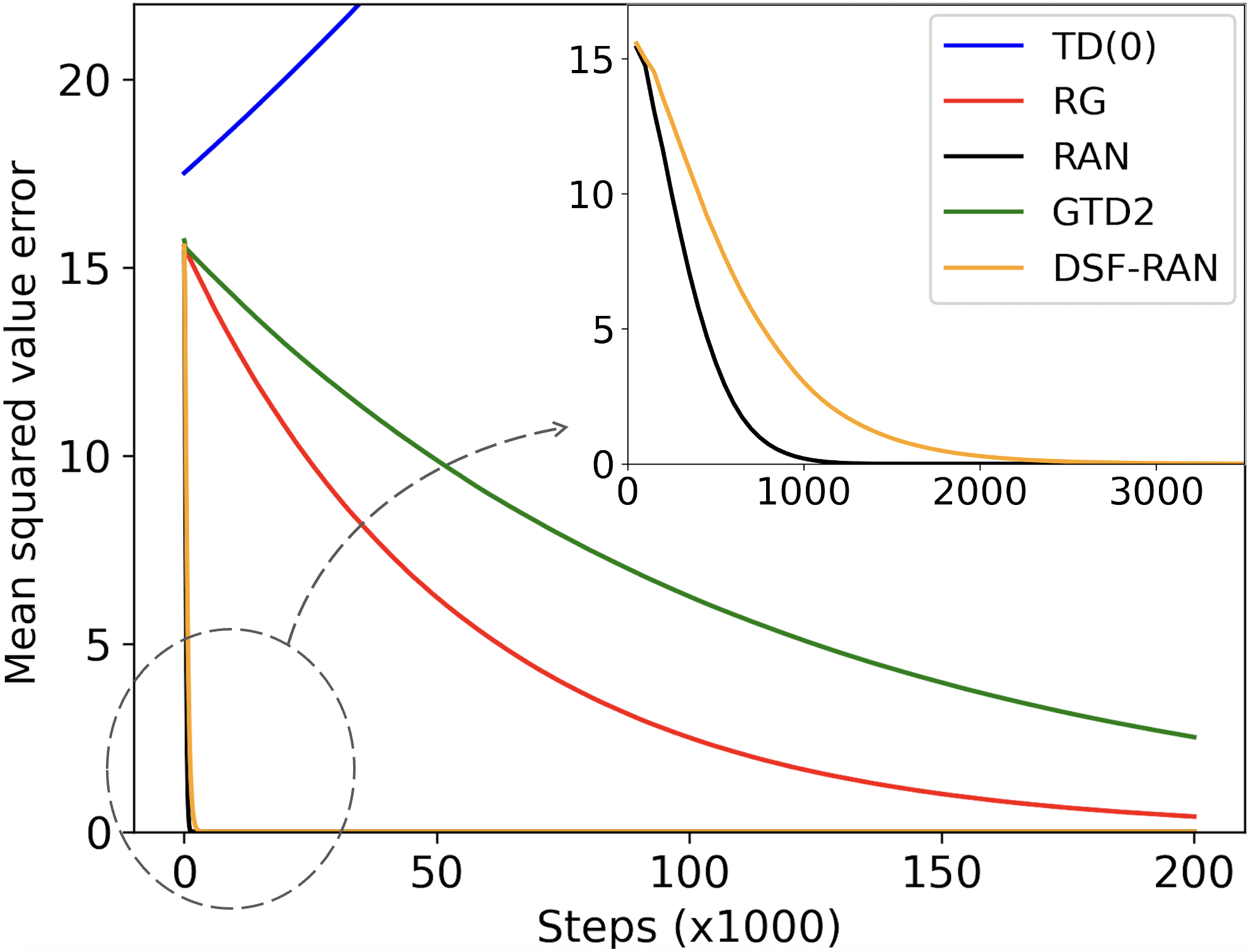}
 	\caption{The Baird's star experiment  discussed in Section~\ref{sec:GTD}}\label{fig:Baird star}
 \end{figure}

\section{The problem of outliers} \label{sec:outlier1} 
In this section we argue that the gradient of MSBE involves large outliers and discuss its impact on the RAN algorithm.
For simplicity, temporarily suppose that the set of actions  is a singleton, $\mathcal{A}=\{a\}$.
In the function approximation case, successive states $S_t$ and $S_{t+1}$ often have similar representations. 
As a result, $\nabla q_\wb(S_t,a)$ and $\gamma \nabla q_\wb(S_{t+1},a)$ are often similar, rendering  $\nabla \delta_t = \gamma \nabla q_\wb(S_{t+1},a) -  \nabla q_\wb(S_t,a)$ to be small \citer{(Zhang et al., 2020)}\cite{ZhanBW2020}. 
This would not have been problematic if $\nabla \delta_t$ was small for all $t$, in which case we could compensate by increasing the step-size. 
However,   $\nabla \delta_t$ can  occasionally be large, for example when $S_{t+1}$ is a terminal state in which case $\nabla \delta_t = - \gamma \nabla q_\wb(S_t,a)$,  or when   $S_{t+1}$ is far from $S_{t}$ (e.g., in large jump transitions). 
Although these \emph{outliers} occur with low probability, they carry important information. For example, the pre-terminal transitions are important because they pin down the estimated values to the terminal values. 
In environments with larger  action sets, if the policy has small entropy,  $A_{t+1}$ and $A_t$ would have similar representations with high probability, causing  $ \nabla \delta_t$ to be small. 

We now discuss how these outliers affect RAN.
The updates of $\mb$ in Algorithm~\ref{alg:ran} involve a
 momentum (of MSBE gradient) term $\lambda\mb +\delta'_t\nabla\delta_t$  and a correction term $-\beta (\nabla\delta_t^T\mb)\nabla\delta_t$ that aims to slowly modify $\mb$ towards the approximate Gauss-Newton direction. 
However, when $\nabla \delta_t$ is an  outlier, $\beta (\nabla\delta_t^T\mb)\nabla\delta_t$ can grow very large, cause an overshoot, and completely change the direction of $\mb$.
In particular, if $\beta \|\nabla \delta_t\|^2>1$, then magnitude of the correction term would be larger than the projection of $\mb$ on $\nabla \delta_t$, i.e.
\begin{equation} \label{eq:overshoot}
\big| \langle \beta (\nabla\delta_t^T\mb)\nabla\delta_t, \nabla\delta_t \rangle     \big|   >  \big| \langle\nabla\delta_t^T \mb\rangle\big|,
\end{equation}
which results in an overshoot along $\nabla \delta_t$. 
Such overshoots hinder $\mb$ from tracking the approximate Gauss-Netwon direction.

To reduce the adverse effect of outliers, one can reduce step-size $\beta$, at the cost of  slowed down learning.
Another popular solution is gradient clipping \citer{(Zhang et al., 2019)}\cite{zhang2019gradient}. 
However, as discussed in the first paragraph of this section, the
outliers  in our problem carry important information, which can be lost via gradient clipping.

\section{Outlier-splitting}\label{sec:outlier2}
We now propose \emph{outlier-splitting} as a general meta-technique for stochastic optimization, appropriate for the case that data contains rare sample functions with abnormally large gradients, and these sample functions carry important information that would be lost in gradient clipping.
We first explain the key idea by an example. Consider minimizing $f_1+\cdots+f_n$  for smooth functions $f_1,\ldots,f_n$. Suppose that $f_1$ is an outlier in the sense that the norm of its gradient is locally $k$ times larger than the gradient norms of  other functions, for some integer $k>1$. The idea is that  instead of applying SGD on $f_1+\cdots+f_n$, we break down $f_1$ into $k$ copies of  $f_1/k$ and apply SGD on $f_1/k+\cdots+f_1/k+f_2+\cdots+f_n$ in a random order. The latter updates are outlier-free while being equivalent to the former updates in expectation. We now proceed to a formal description.

Consider SGD on an objective function $F=\Exp[f]$.  
For any sample sample function $f$ and any point $\wb$, we consider a non-negative measure $\xi(f,\wb)$% for  the abnormality of $f$ at $\wb$
; e.g., $\xi(f,\wb)=\|\nabla f(\wb)\|$ or $\|\nabla f(\wb)\|^2$.
Let $\bar{\xi}$ be a trace of $\xi$, updated by $\bar{\xi}\leftarrow \lambda_\xi\bar{\xi} + (1-\lambda_\xi)\xi(f_t,\wb_t)$, where $\lambda_\xi\in (0,1)$ is a constant close to $1$.
We say that $f_t$ is an outlier if $\xi(f_t,\wb_t)\ge \rho \bar{\xi}_t$, for some \emph{outlier threshold} $\rho>1$.

The pseudo code of the outlier-splitting method for online SGD is given in Algorithm~\ref{alg:outlier-splitting}. 
At time $t$ of this algorithm, we let 
\begin{equation}\label{eq:k}
k=\left\lfloor \frac{\xi(f_t,\wb_t)}{\rho \bar{\xi}_t} \right\rfloor+1. 
\end{equation}
If $f_t$ is an outlier (equivalently $k>1$),
instead of $f_t$ we pretend to have $k$ copies of $f_t/k$. 
We use one of these copies to do a gradient update at time $t$, and store the remaining $k-1$ copies in a buffer to use them for future updates. 
These copies are stored in one cell of an \emph{outlier-buffer} as a tuple $(f_t,k,k-1)$, where  $k-1$ indicates the number of remaining copies to be used for future updates.
In each iteration we perform one update based on the online sample, and perform at most one update based on a sample from the buffer.
More concretely, in each iteration $t$, after applying a gradient update  $\wb\leftarrow \wb-(\beta/k)\nabla f_t(\wb)$, we take a sample $(f,k_f,j)$ from the outlier buffer with some positive probability, and perform a gradient update  $\wb\leftarrow \wb-(\beta/k_f)\nabla f(\wb)$.

\begin{algorithm}[tb]
	\caption{\quad Outlier-splitting for online SGD, applied to loss function $F=\Exp[f]$}
	\label{alg:outlier-splitting}
	\begin{algorithmic}
		\STATE {\bfseries Parameters:} step-size $\beta$, outlier threshold $\rho$, trace parameter $\lambda_\xi$, outlier sampling probability $\sigma$.
		\STATE {\bfseries Initialize:} $\hat{\xi}=0$, $\wb$.
		\FOR{$t=1,2,\ldots$}
		%\STATE compute $\xi(f_t,\wb)$
		\STATE $\hat{\xi}\leftarrow  \lambda_\xi\hat{\xi}+(1-\lambda_\xi)\xi(f_t,\wb)$
		\STATE $\bar\xi =  \hat{\xi} / (1-\lambda_\xi^t)$     \hspace{0cm} \algcomment{\scriptsize bias-corrected trace estimate}
		\STATE $k=\lfloor \xi(f_t,\wb)/(\rho\bar{\xi}) \rfloor +1$
		\STATE $\wb\leftarrow \wb-(\beta/k)\nabla f_t(\wb)$
		%\STATE {\bfseries if}{ $k>1$}{\bfseries:} store  $(f,k,k-1)$ in the outlier buffer
		\IF{$k>1$}
		\STATE Store  $(f,k,k-1)$ in the outlier buffer
		\ENDIF
		\STATE {\bfseries With probability}  $\min(1, \sigma*\mbox{length of outlier bufffer})${\bfseries:}
		\bindent
		\STATE Sample $(f,k',j)$ uniformly form outlier buffer
		\STATE  $k''=\max\big(k', \lfloor \xi(f,\wb)/(\rho\bar{\xi}) \rfloor +1\big)$
		\STATE $\wb\leftarrow \wb-(\beta/k'')\nabla f(\wb)$
		\IF{$j>1$}
		\bindentb
		\STATE Replace $(f,k',j)$ with $(f,k',j-1)$ in the  buffer
		\eindent
		\ELSIF{$j=1$}
		\bindentb
		\STATE Remove $(f,k',j)$ from the outlier buffer
		\eindent
		\ENDIF
		\eindent
		\ENDFOR
	\end{algorithmic}
\end{algorithm}

We now show that the outlier buffer is  stable.  %Since the outlier threshold $\rho$ is larger than one, 
The expected number of copies, $k-1$,  added to the buffer at time $t$ satisfies
\begin{equation*}
   \Exp[k-1]\le\Exp_t\left[  \frac{\xi(f_t,\wb_t)}{\rho \bar{\xi}_t} \right] \simeq  \frac{\Exp_t\left[  \xi(f_t,\wb_t) \right]}{\rho \Exp[\bar{\xi}_t]}=\frac1{\rho}<1,
\end{equation*}
where the inequality is due to \eqref{eq:k} and the approximate equality is because $\bar{\xi}_t$ is a long-time average.
On the other hand, as the length of the outlier buffer increases, the probability of performing a sample update from the buffer goes to $1$. In this case, arrival rate to the buffer, $1/\rho$, is smaller than its departure rate, $1$; implying 
 stability of the outlier buffer. 

\section{Our main algorithm: RANS} \label{sec:rans}
\todonext{The experiments are using a "before" version of RANS. I also ran experiments with the "after" versions. The after version were better in Acrobot and were worse in Cartpole. The difference was small though. Decided to keep the before version (because better curves and simpler formula). But for more complex environments, I have to try both.)}
Our final algorithm, RAN with outlier Splitting (RANS), is a combination of RAN, outlier-splitting, and adaptive step-size ideas.
In order to improve updates of $\mb$, we employ  an adaptive vector step-size  $\betab$ that evolves according to a mechanism quite similar to RMSProp \citer{(Kochenderfer \& Wheeler, 2019)}\cite{kochenderfer2019optBook}, as we discuss next.

Consider  a trace vector $\nub_t$ of $(\nabla\delta_t)^2$ updated according to 
\begin{equation*}
\nub_t\leftarrow \lambda'\nub_{t-1}+(1-\lambda')(\nabla\delta_t)^2,
\end{equation*}
where $(\nabla\delta_t)^2$ is the entrywise square vector of $\nabla\delta_t$, and $\lambda'\in[0,1)$ is a constant. We consider an outlier-measure 
\begin{equation}\label{eq:xi}
\xi_t = \langle\frac1{\sqrt{\nub_t}}\odot \nabla\delta_t, \nabla\delta_t\rangle
\end{equation}
where $1/{\sqrt{\nub_t}}$  is entrywise square root, and $\odot$ and $\langle\cdot,\cdot\rangle$ denote entrywise product and inner product of two vectors, respectively.
We then compute the trace $\bar{\xi}$ and $k$ as in Section~\ref{sec:outlier2}:
$\bar{\xi}_t\leftarrow \lambda'\bar{\xi}_t+(1-\lambda')\xi_t$ and $k=\left\lfloor {\xi_t}/{(\rho \bar{\xi}_t)} \right\rfloor+1$.
We finally fix an $\eta\in(0,1)$ and choose the step-size 
\begin{equation}\label{eq:betab}
\betab_t=\frac{\eta}{\rho \bar{\xi}_t} \, \frac1{\sqrt{\nub_t}}.
\end{equation}

The pseudo code of RANS is given in Algorithm~\ref{alg:rangs}  in Appendix~\ref{app:rans}. 
The algorithm involves applying the outlier-splitting method on the updates of $\mb$ in RAN, and using the adaptive step-size in  \eqref{eq:betab}. 

We now shows that the outlier-splitting mechanism in RANS effectively prevents overshoots of type \eqref{eq:overshoot} in the updates of $\mb$. 
Given the above choice of $\betab_t$, we have
\begin{equation*}
\begin{split}
\frac1k  \langle\betab_t\odot &\nabla\delta_t, \nabla\delta_t\rangle = \frac1k\,\frac{\eta}{\rho \bar{\xi}_t}  \langle\frac1{\sqrt{\nub_t}}\odot \nabla\delta_t, \nabla\delta_t\rangle \\
&\le \frac{\rho \bar{\xi}_t}{\xi_t}\,\frac{\eta}{\rho \bar{\xi}_t} \, \langle\frac1{\sqrt{\nub_t}}\odot \nabla\delta_t, \nabla\delta_t\rangle 
=\eta,
\end{split}
\end{equation*}
where the first equality is from the definition of $\betab_t$ in \eqref{eq:betab}, the inequality is due to the definition of $k$, and the last equality follows from the definition of $\xi_t$ in \eqref{eq:xi}.  
This implies that
\begin{equation}\label{eq:no overshoot RANS}
\big| \langle \frac{1}k \betab (\nabla\delta_t^T\mb)\nabla\delta_t, \nabla\delta_t \rangle     \big|   \le \eta \big|\nabla\delta_t^T \mb\big|.
\end{equation}
Therefore  overshoots of type \eqref{eq:overshoot} do not occur in RANS.

The RANS algorithm has hyperparameters $\alpha$, $\eta$, $\rho$, $\lambda$, $\lambda'$, and $\sigma$ (the outlier sampling probability). 
Setting  $\eta=0.2$ and $\rho=1.2$ are always good choices. 
Furthermore, our experiments show that the parameters $\lambda$,  $\lambda'$, and $\sigma$ can be set to the default values  $\lambda=0.999$,  $\lambda'=0.9999$, and $\sigma=0.02$ without much performance degradation. 
In this case, the RANS algorithm would  have essentially one hyper-parameter $\alpha$, just like  RG and TD algorithms with Adam optimizer~\citer{(Kingma \& Ba, 2014)}\cite{kingma2014adam}.  
The per-iteration computational complexity of RANS  is at most twice the  RG algorithm with Adam optimizer.

\section{Experiments} \label{sec:experiments}
Similar to TD, the RANS algorithm can be utilized within any control loop. More specifically, one can used RAN instead of TD to estimate $Q$-values and plug these estimates into the policy update of interest, including actor-critic algorithms like A3C \cite{mnih2016A3C}\citer{(Minh et al., 2016)}, deterministic policy gradient algorithms  \cite{silver2014DPG}\citer{(Silver et al., 2014)} like DDPG \cite{Lillicrap2015DDPG}\citer{(Lillicrap et al., 2015)}, and greedy/soft-max policy updates like DQN \cite{mnih2015DQN}\citer{(Minh et al., 2015)}. 
In this section, we assess the performance of softmax policy updates  and deterministic policy gradient methods when the $Q$-functions in these algorithms are calculated using RANS.  

For softmax policy updates, we conducted experiments on Acrobot and Cartpole environments. We used a single-layer neural network with 64 hidden units with ReLU activation to learn the action-values via three algorithms, TD(0), RG, and RANS. The actions are chosen according to a softmax distribution on the action-values.
Fig.~\ref{fig:control} illustrates expected returns versus number of step.
We trained TD(0) and RG using Adam optimizer. 
Refer to Appendix~\ref{app:exp control} for complementary experimental results and details of the experiments.
The results show that the RANS algorithm outperforms RG and TD on these environments.
% This suggests that RANS has the potential  for competitive performance in more challenging tasks like MuJoCo environments. 

\begin{figure*}[t!]
	\centering % the figures are generated by ComputeCanada_compatible/z_2experiment_2TimeScale_Q_values
	\begin{subfigure}[t]{0.4\linewidth}
		\centering
		\includegraphics[width=1\linewidth]{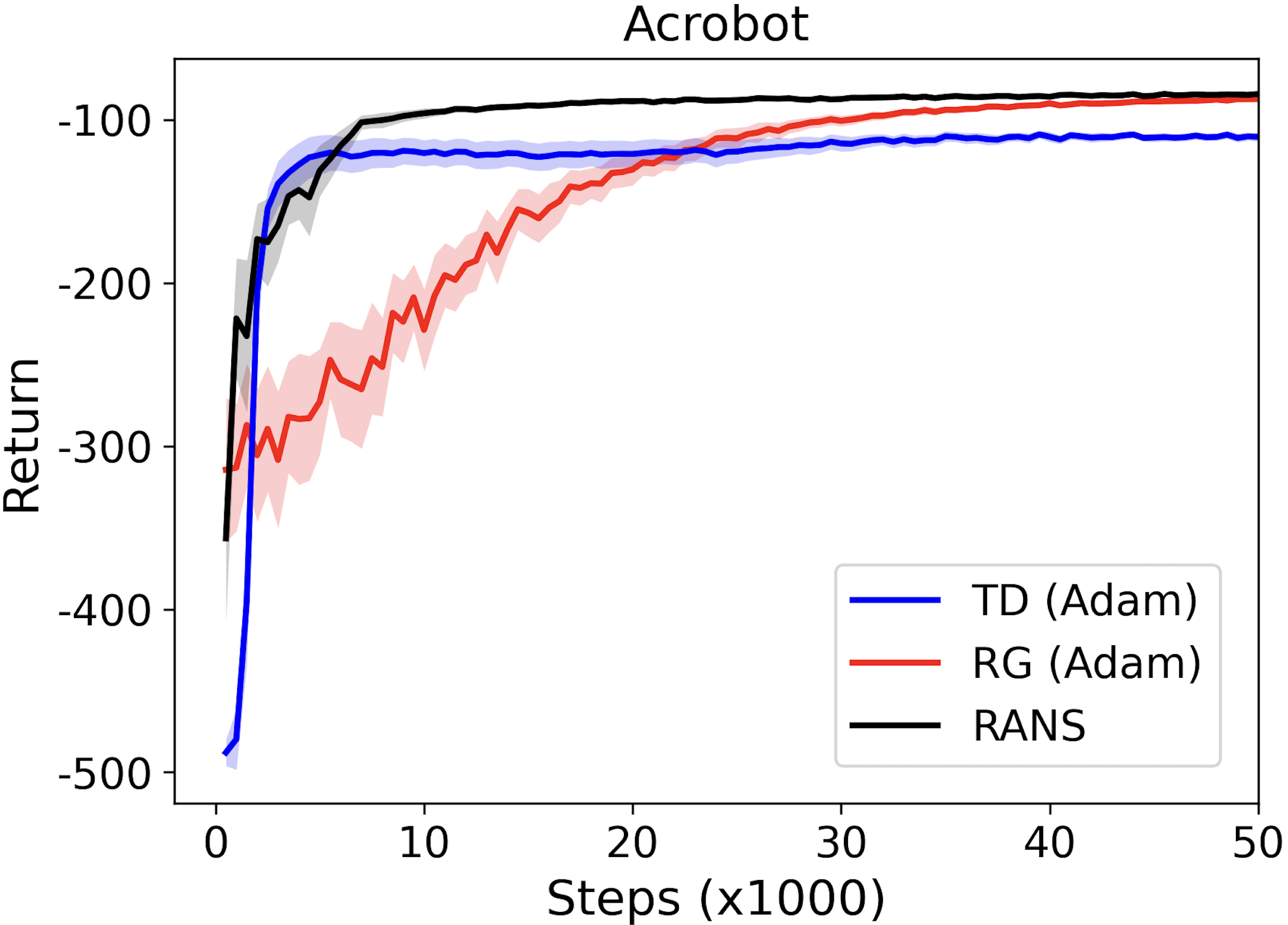}
	\end{subfigure}%
	\qquad\quad
	\vspace{.3cm}
	\begin{subfigure}[t]{.39\linewidth}
		\centering
		\includegraphics[width=1\linewidth]{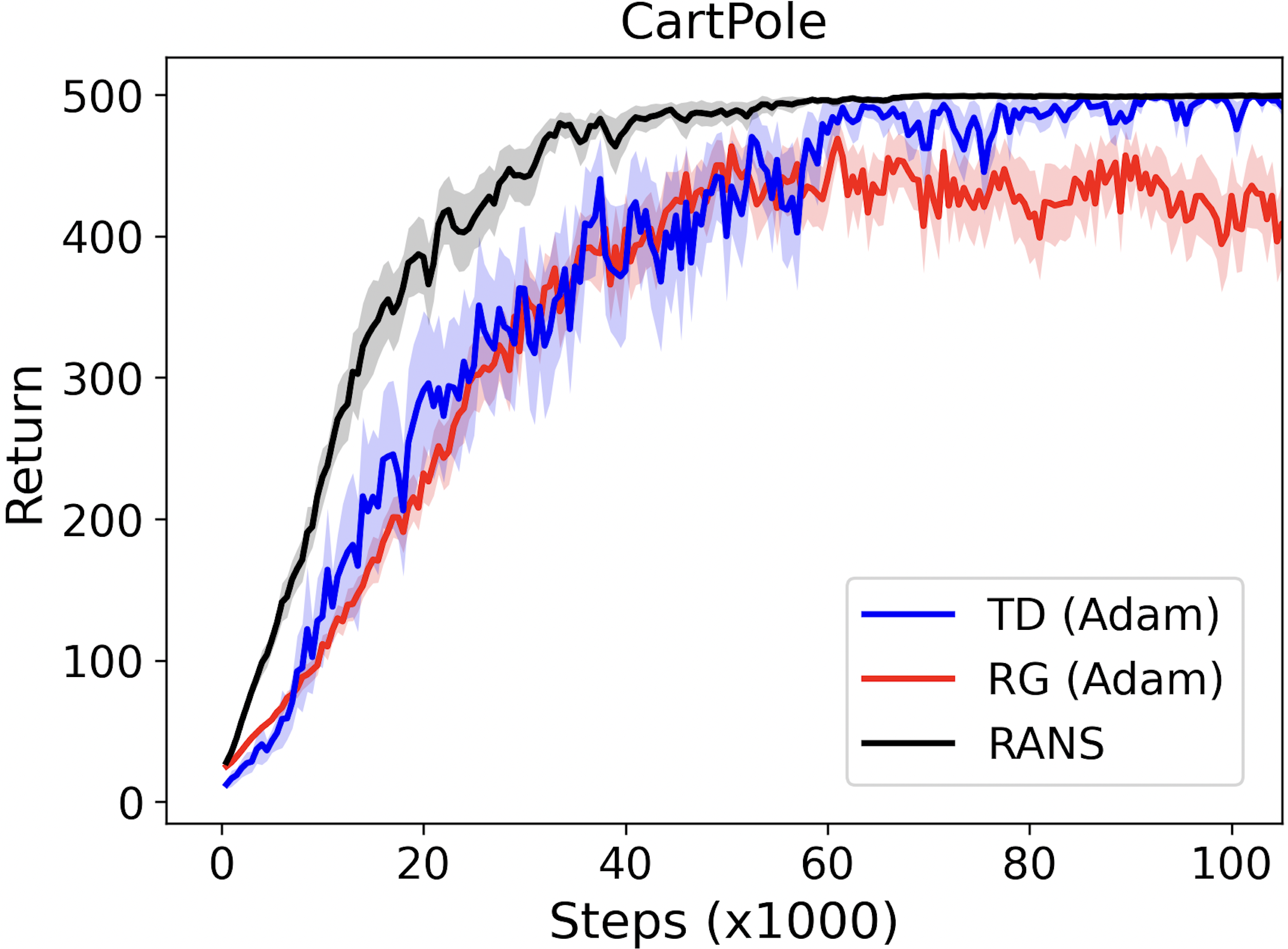}
	\end{subfigure}
	\caption{Performance of RANS, TD(0), and RG on classic control tasks. A single-layer neural network with 64 hidden ReLU units  was used to learn the $Q$-values, and a softmax distribution on the $Q$-values was used as the policy. }\label{fig:control}
	%\end{center}
\end{figure*}

For the deterministic policy gradient  actor updates, we conducted experiments on MuJoCo environments: Hopper and HalfCheetah.
	We employed a two-layer feedforward neural network with 400 and 300 hidden ReLU units respectively for both the actor and critic, and a final tanh unit following the output of the actor. The actor was trained using deterministic policy gradient policy updates \cite{silver2014DPG,Lillicrap2015DDPG}\citer{(Silver et al., 2014; Lillicrap et al., 2015)}, while the critics were trained by three algorithms: RANS, Adam TD with delayed target network update, and Adam RG. We considered an on-policy setting where samples are drawn from the current policy in an online manner and are directly fed into the actor and critic training algorithms. We did not use replay buffers or batch updates.
Fig.~\ref{fig:MuJoCo} depicts the learning curves of these algorithms. Refer to Appendix~\ref{app:exp control} for additional experimental results and details of the experiments. The results indicate that the RANS algorithm surpasses TD and RG in these environments. It is important to note that the results cannot be fairly compared to the state-of-the-art because our setting is on-policy and does not take advantage of replay buffers and batch updates. We leave the integration of these techniques into the RANS ideas for future work.

\begin{figure*}[t!]
	\centering % the figures are generated by 5_from_gcloud
	\begin{subfigure}[t]{0.39\linewidth}
		\centering
		\includegraphics[width=1\linewidth]{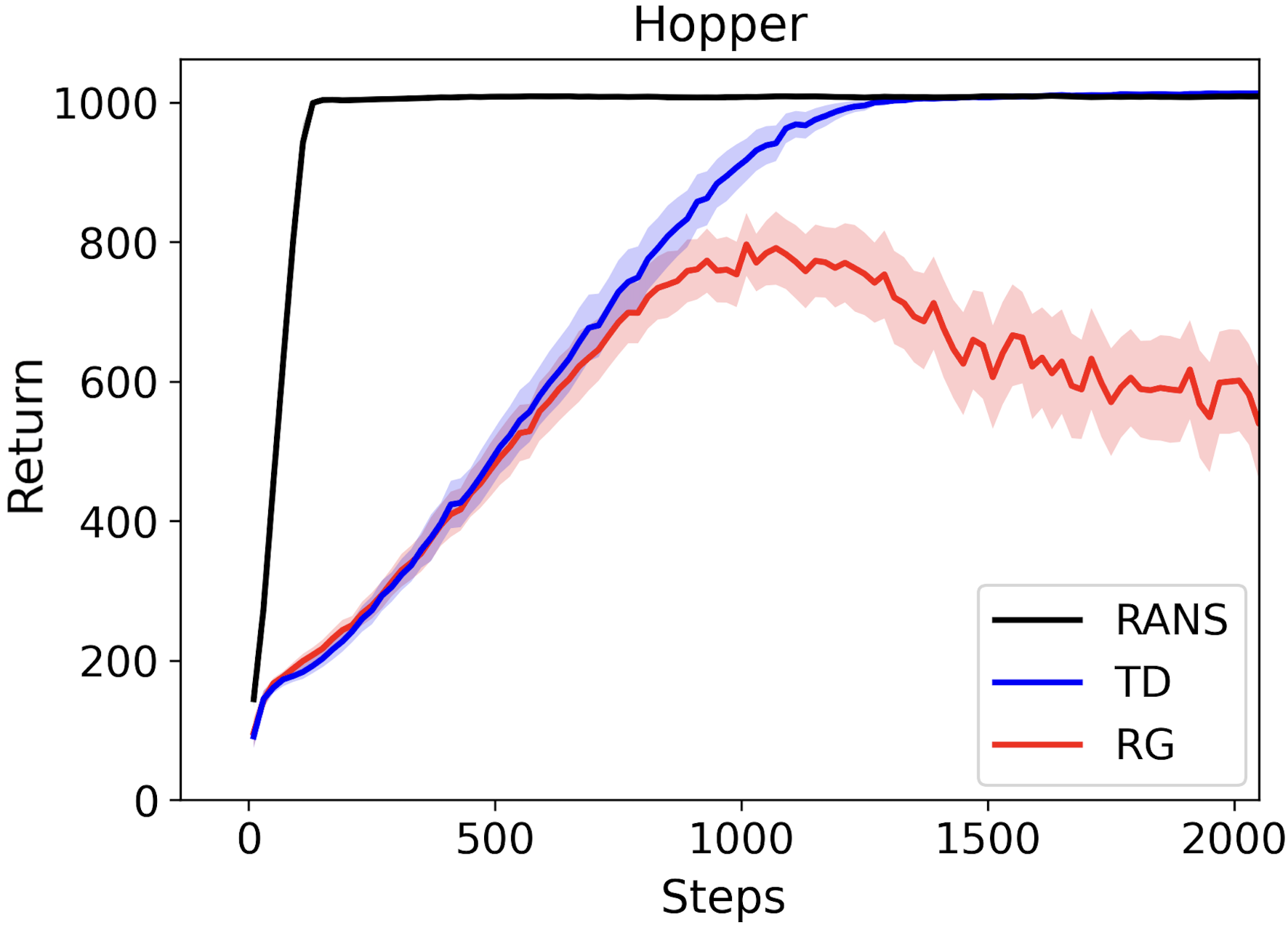}
	\end{subfigure}%
	\qquad\quad
	\begin{subfigure}[t]{.39\linewidth}
		\centering
		\includegraphics[width=1\linewidth]{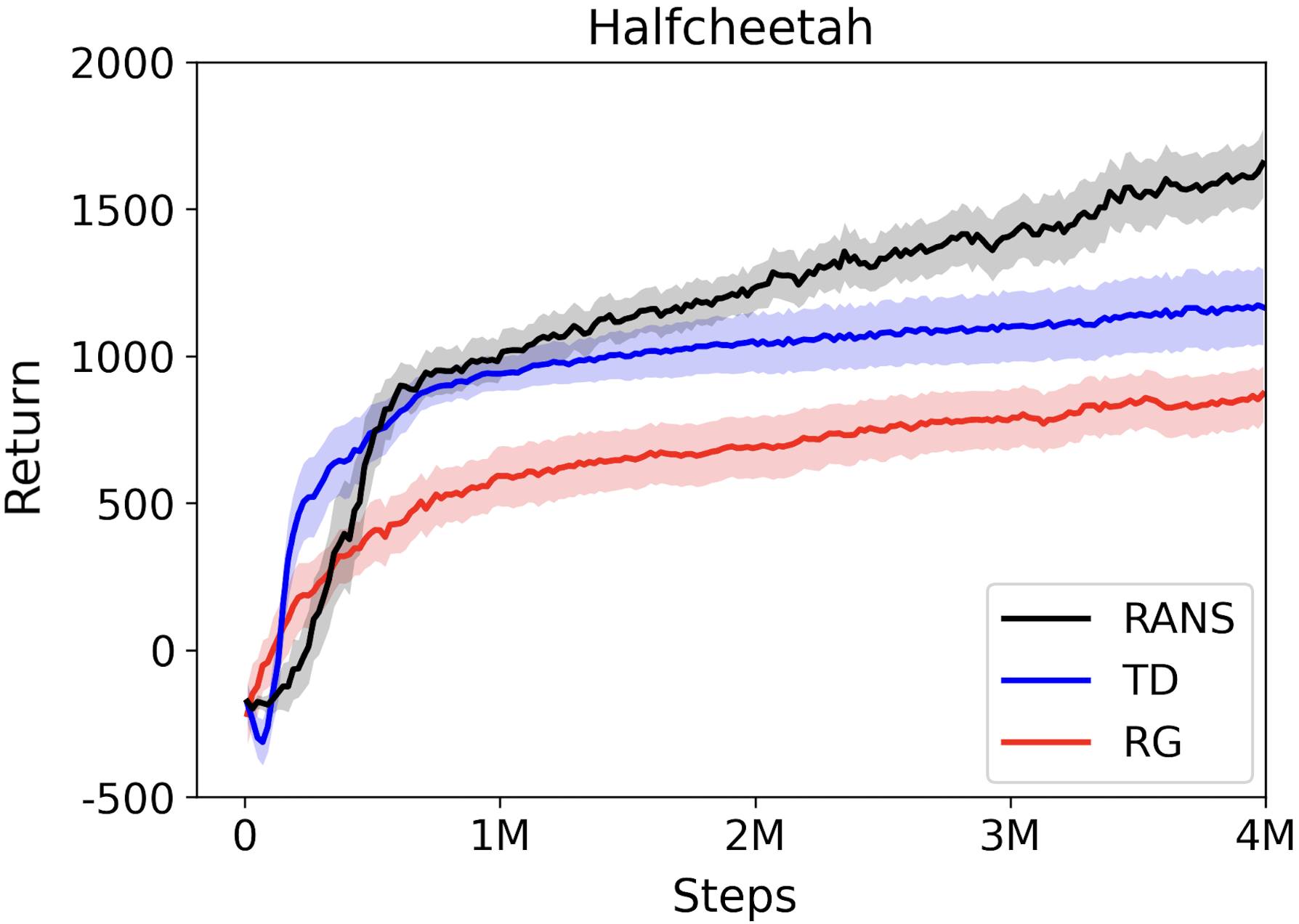}
	\end{subfigure}
	\caption{Performance of RANS, TD with target networks, and RG algorithms on simple MuJoCo environments. The $Q$-values, trained via these algorithms, were integrated into a standard deterministic policy gradient control loop for policy training. The experiments are online, not utilizing replay buffers or batch updates.}\label{fig:MuJoCo}
	%\end{center}
\end{figure*}

\section{Related works}
Poor conditioning of MSBE was previously observed in \citer{(Wang \& Ueda, 2021)}\cite{WangU21} through study of an example Markov chains. 
%\cite{WangU21} studied some examples of Markov chains and showed that MSBE is ill-conditioned in those examples. 
More specifically, \citer{Wang and Ueda (2021)}\cite{WangU21} analyzed a particular $n$-state Markov chain and showed that the condition-number of MSBE in this Markov chain scales with $n^2$.  
They also showed that the condition-number scales with $1/(1-\gamma)^2$ in another example Markov chain. In comparison, our lower bound in Theorem~\ref{th:cond}~(a) holds for every Markov chain, and the lower bound in Theorem~\ref{th:cond}~(b) scales with $n^2/(1-\gamma)^2$.
%\citer{Deb and Bhatnagar (2022)}\cite{deb2022gradient} studied the effect of momentum on the  performance of gradient-TD algorithms, and reported modest improvement over vanilla versions.

A prevalent explanation for slowness of gradient-based value estimation methods is the so called information flow in the wrong direction  \citer{(Baird, 1995)}\cite{Bair95}. More concretely, each update in RG can be decomposed into a forward bootstrapping component (or a TD update) and a backward bootstrapping component  (the so called wrong direction of information flow). 
A common approach for accelerating the gradient updates is by suppressing the second component (e.g., via some sort of combination with TD updates), especially in early  stages of training.
The acceleration gained in the residual algorithm \citer{(Baird, 1995)}\cite{Bair95}, TDC \citer{(Sutton et al., 2009)}\cite{SuttMPBSSW09}, TDRC, and QRC \citer{(Ghiassian et al., 2020)}\cite{Ghiassian2020TDRC} 
can be understood from this perspective. 
In contrast, acceleration gained in our algorithms does not rely on combinations with TD updates.

Use of Gauss-Newton method for value estimation was explicitly proposed in \citer{(Gottwald et al., 2021; Gottwald \& Shen, 2022)}\cite{GottGSD21,GottS22}, recently.
Value estimation algorithms based on Kalman filter \citer{(Choi \& Van Roy, 2006; Geist \& Pietquin, 2010)}\cite{choi2006kalman,geist2010kalman} are also known to have an equivalent form to online Gauss-Newton updates \citer{(Geist \& Pietquin, 2010)}\cite{geist2010kalman}.
\citer{Sun and Bagnell (2015)}\cite{sun2015newton} studied  MSBE minimization with Newton method.
However, all of the above methods involve approximating a variant of the Hessian or Gauss-Newton matrices and solving a system of linear equations in each iteration, which is  computationally costly.

\citer{Yao et al. (2009)}\cite{yao2009} proposed a low complexity two time-scale method, called LMS-2, for stochastic linear regression. Our RAN algorithm can be perceived as a generalization of the LMS-2 algorithm to MSBE minimization under non-linear function approximation. 
Several other algorithms including least squares TD \citer{(Sutton \& Barto, 2018)}\cite{SuttB18} and \citer{(Devraj \& Meyn, 2017)}\cite{devraj2017zap} also leverage matrix gain for improved convergence, under linear function approximation.

%{\bf Natural gradient:}
In the same spirit, natural gradient methods \citer{(Amari, 1998; Kakade, 2001; Martens, 2020)}\cite{amari1998natural,kakade2001natural,martens2020natural} also enjoy robustness to parameterization.   \citer{Dabney and Thomas (2014), Knight and  Lerner (2018), and Achiam et al. (2019)}\cite{DabnT14,knight2018natural,achiam2019} proposed natural gradients algorithms for value estimation. 
\citer{Dabney and Thomas (2014)}\cite{DabnT14} also proposed a  low complexity two time scale implementation  that has high-level algorithmic similarities to the RAN algorithm.

%{\bf Proximal methods:}
In Section~\ref{sec:alg1} and Appendix~\ref{app: proximal view of RAN} we showed that the RAN algorithm can be perceived as a proximal method.
A proximal method for value estimation, called GTD2-MP, was proposed in \citer{(Liu et al., 2020; Mahadevan et al., 2014)}\cite{LiuGMP20,mahadevan2014proximal}.
However, these works consider a Bregman divergence that does not depend on the value estimates.
In fact, as the step-size goes to zero,  update direction of GTD2-MP tends to  the expected GTD2 update direction. 
\citer{Schulman et al. (2015), Sun and  Bagnell (2015), and Zhu and Murray (2022)}\cite{schulman2015GAE,sun2015newton,zhu2022gradient} considered proximal methods with value dependent penalties  of the form $E[(v_{\wb_{t+1}}(S_t)-v_{\wb_t}(S_t))^2]$.  %which leads to the expected update direction $\Exp_s[\nabla v_\wb(s) \nabla v_\wb(s)^T]^{-1}\nabla MSBE(\wb)$. 
Although the resulting expected update direction $\Exp_s[\nabla v_\wb(s) \nabla v_\wb(s)^T]^{-1}\nabla MSBE(\wb)$ is robust to parameterization, 
it is not robust against poor conditioning. 
For example, in the tabular case, this expected update direction simplifies to $\nabla MSBE(\wb)$, which is the same as RG. 
In contrast, in the proximal view of the RAN algorithm, we used penalties of type $E[(\delta_{\wb_{t+1}}(S_t,A_t)-\delta_{\wb_t}(S_t,A_t))^2]$, which provides robustness to the conditioning of MSBE, as discussed in Section~\ref{sec:alg1} and Appendix~\ref{app: proximal view of RAN}.

The control algorithm SBEED \cite{dai2018sbeed}\citer{(Dai et al., 2018)} involves mirror descent RG  along with  some  other ideas including  entropy regularization, akin to SAC \cite{haarnoja2018SAC}\citer{(Haarnoja et al., 2018)}, and blending RG with \emph{naive residual gradient} \cite{SuttB18}\citer{(Sutton and Barto, 2018, Chapter~11)}. The entropy regularization technique, in particular, is known to produce significant performance gains. It is noteworthy that the  entropy regularization technique can be used in conjunction with RANS and DFS-RAN algorithms, as well.

\citer{Karampatziakis and Langford (2010) and Tian and Sutton (2019)}\cite{KaraL10,TianS19} proposed a method, called \emph{sliding-step}, to reduce the adverse effect of outliers in certain problems. 
This method is pretty similar to the outlier-splitting algorithm, with the only difference that in the sliding-step method, all $k$ updates $\wb\leftarrow\wb-\nabla f_t(\wb)/k$ are applied sequentially and before time $t+1$, while the outlier-splitting method spreads these updates over a long time.
Another simple approach is using momentum to reduce the variance of updates. However, smoothing large outliers requires large momentum parameters, in which case the delayed effect of gradients propagate far into future and become out-dated, pushing $\wb$ along outdated outlier gradient even if the  outlier gradient at current $\wb$ is reversed.  %, leading to performance degradation.

\section{Future works and discussion}
In this paper, we highlighted causes that underlie slowness of gradient-based value estimation methods, and proposed low complexity techniques to resolve them.
Our focus was on the on-policy case, however the proposed algorithms are easily applicable for off-policy learning when combined with standard importance sampling techniques.
We provided evidence for the potential of the proposed algorithms via experiments on a few classic environments.

Other than applying standard techniques (such as batch updates, replay buffers, different forms of step-size adaptation, etc.) and testing the algorithms on more complex environments, there are several directions for future research. This includes adopting the unbiased gradient estimate of \eqref{eq:quad loss} in \eqref{eq:2TS delta'_T} instead of the biased estimate  in \eqref{eq:2TS}, and  comparing these methods with other means of solving  \eqref{eq:quad loss}, including conjugate gradient and low rank approximation of the Gauss-Newton matrix.
Another important direction is further exploration of the proposed double-sampling-free methods in stochastic environments with neural network function approximation.
On the theory side, it would be interesting to study condition-number of MSBE, and in general the shape of MSBE landscape, under linear and non-linear function approximation under common feature representations in asymptotically large environments.
Moreover, non-asymptotic behavior and finite sample complexity analysis of the proposed methods would be very helpful for understanding the effectiveness of these algorithms in reducing sensitivity to condition-number.

\section{Acknowledgments}
The authors want to thank Yi wan, Sina Ghiassian, John N. Tsitsiklis, and Saber Salehkaleybar for their valuable feedback in various stages of development of this work.

%%%%%%%%%%%%%%%%%%%%%%%%%%%%%%%%%%%%%%
%%%%%%%%%%%%%%%%%%%%%%%%%%%%%%%%%%%%%%

\ifbibtex
\bibliography{RL}
\bibliographystyle{icml2023}
\else
\citer{
\section*{References}
	
\begin{list}{}{%
	\setlength{\topsep}{0pt}%
	\setlength{\leftmargin}{0.2in}%
	\setlength{\listparindent}{-0.2in}%
	\setlength{\itemindent}{-0.2in}%
	\setlength{\parsep}{\parskip}%
}%
%\small
\item[]
Achiam, J., Knight, E., and Abbeel, P. Towards characterizing divergence in deep Q-learning. \emph{arXiv preprint arXiv:1903.08894}, 2019.

Amari, S. I. Natural gradient works efficiently in learning. \emph{Neural Computation}, 10(2):251–276, 1998.

Baird, L. Residual algorithms: Reinforcement learning with function approximation. In \emph{Proceedings of the International Conference on Machine Learning}, pp. 30–37. 1995.

Barnard, E. Temporal-difference methods and Markov models. \emph{IEEE Transactions on Systems, Man, and Cybernetics}, 23(2):357–365, 1993.

Bertsekas, D. and Tsitsiklis, J. N. \emph{Neuro-dynamic programming}. Athena Scientific, 1996.

Bhatnagar, S., Sutton, R. S., Ghavamzadeh, M., and Lee, M. Natural actor-critic algorithms. \emph{Automatica}, 45, 2009.

Boyan, J. A. Technical update: Least-squares temporal difference learning. \emph{Machine Learning}, 49(2):233–246, 2002.

Brandfonbrener, D. and Bruna, J. Geometric insights into the convergence of nonlinear TD learning. \emph{arXiv preprint arXiv:1905.12185}, 2019.

Choi, D. and Van Roy, B. A generalized Kalman filter for fixed point approximation and efficient temporal-difference learning. \emph{Discrete Event Dynamic Systems}, 16(2):207–239, 2006.

Dabney, W. and Thomas, P. Natural temporal difference learning. In \emph{Proceedings of the AAAI Conference on Artificial Intelligence}, volume 28, 2014.

Dai, B., He, N., Pan, Y., Boots, B., and Song, L. Learning from conditional distributions via dual embeddings. In \emph{Proceedings of the International Conference on Artificial Intelligence and Statistics}, pp. 1458–1467, 2017.

Deb, R. and Bhatnagar, S. Gradient temporal difference with momentum: Stability and convergence. In \emph{Proceedings of the AAAI Conference on Artificial Intelligence}, volume 36, pp. 6488–6496, 2022.

Devraj, A. M. and Meyn, S. Zap Q-learning. In \emph{Advances in Neural Information Processing Systems}, volume 30, 2017.

Du, S., Chen, J., Li, L., Xiao, L., and Zhou, D. Stochastic variance reduction methods for policy evaluation. In \emph{Proceedings of the International Conference on Machine Learning}, pp. 1049-1058, 2017.

Geist, M. and Pietquin, O. Kalman temporal differences. \emph{Journal of artificial intelligence research}, 39:483–532, 2010.

Ghiassian, S. and Sutton, R. S. An empirical comparison of off-policy prediction learning algorithms in the four rooms environment. \emph{arXiv preprint arXiv:2109.05110}, 2021.

Ghiassian, S., Patterson, A., Garg, S., Gupta, D., White, A., and White, M. Gradient temporal-difference learning with regularized corrections. In \emph{Proceedings of the International Conference on Machine Learning}, pp. 3524–3534, 2020.

Gordon, G. J. \emph{Approximate solutions to Markov decision processes}. Ph.D. thesis, Carnegie Mellon University, 1999.

Gottwald, M. and Shen, H. On the compatibility of multistep lookahead and hessian approximation for neural residual gradient. In \emph{Proceedings of the Multi-disciplinary Conference on Reinforcement Learning and Decision Making}, 2022.

Gottwald, M., Gronauer, S., Shen, H., and Diepold, K. Analysis and optimization of Bellman residual errors with neural function approximation. \emph{arXiv preprint arXiv:2106.08774}, 2021.

Hackman, L. M. \emph{Faster Gradient-TD Algorithms}. M.Sc. thesis, University of Alberta, Edmonton, 2012.

Juditsky, A. and Nemirovski, A. \emph{Optimization for Machine Learning}. MIT Press, 2011.

Kakade, S. M. A natural policy gradient. In \emph{Advances in Neural Information Processing Systems}, volume 14, 2001.

Karampatziakis, N. and Langford, J. Online importance weight aware updates. \emph{arXiv preprint arXiv:1011.1576}, 2010.

Kingma, D. P. and Ba, J. Adam: A method for stochastic optimization. \emph{arXiv preprint arXiv:1412.6980}, 2014.

Knight, E. and Lerner, O. Natural gradient deep Q-learning. \emph{arXiv preprint arXiv:1803.07482}, 2018.

Kochenderfer, M. J. and Wheeler, T. A. \emph{Algorithms for optimization}. MIT Press, 2019.

Konda, V. and Tsitsiklis, J. Actor-critic algorithms. In \emph{Advances in Neural Information Processing Systems}, volume 12, 1999.

Kushner, H. and Yin, G. G. \emph{Stochastic approximation and recursive algorithms and applications}. Springer Science \& Business Media, 2003.

Lillicrap, T. P., Hunt, J. J., Pritzel, A., Heess, N., Erez, T., Tassa, Y., Silver, D., and Wierstra, D. Continuous control with deep reinforcement learning. \emph{arXiv preprint arXiv:1509.02971}, 2015.

Liu, B., Liu, J., Ghavamzadeh, M., Mahadevan, S., and Petrik, M. Finite-sample analysis of proximal gradient td algorithms. \emph{arXiv preprint arXiv:2006.14364}, 2020.

Macua, S. V., Chen, J., Zazo, S., and Sayed, A. H. Distributed policy evaluation under multiple behavior strategies. \emph{IEEE Transactions on Automatic Control}, 60(5): 1260–1274, 2014.

Maei, H. R. \emph{Gradient temporal-difference learning algorithms}. Ph.D. thesis, University of Alberta, Edmonton, 2011.

Maei, H. R., Szepesv{\'a}ri, C., Bhatnagar, S., and Sutton, R. S. Toward off-policy learning control with function approximation. In \emph{Proceedings of the 27th International Conference on Machine Learning}, pp. 719–726, 2010.

Mahadevan, S., Liu, B., Thomas, P., Dabney, W., Giguere, S., Jacek, N., Gemp, I., and Liu, J. Proximal reinforcement learning: A new theory of sequential decision making in primal-dual spaces. \emph{arXiv preprint arXiv:1405.6757}, 2014.

Marquardt, D. W. An algorithm for least-squares estimation of nonlinear parameters. \emph{Journal of the Society for Industrial and Applied Mathematics}, 11(2):431–441, 1963.

Martens, J. New insights and perspectives on the natural gradient method. \emph{Journal of Machine Learning Research}, 21(1):5776–5851, 2020.

Mnih, V., Kavukcuoglu, K., Silver, D., Rusu, A. A., Veness, J., Bellemare, M. G., Graves, A., Riedmiller, M., Fidjeland, A. K., Ostrovski, G., et al. Human-level control through deep reinforcement learning. \emph{Nature}, 518(7540): 529–533, 2015.

Mnih, V., Badia, A. P., Mirza, M., Graves, A., Lillicrap, T., Harley, T., Silver, D., and Kavukcuoglu, K. Asynchronous methods for deep reinforcement learning. In \emph{Proceedings of the International Conference on Machine Learning}, pp. 1928–1937, 2016.

Nesterov, Y. and Polyak, B. T. Cubic regularization of newton method and its global performance. \emph{Mathematical Programming}, 108(1):177–205, 2006.

Nocedal, J. and Wright, S. J. \emph{Numerical optimization}. Springer, 1999.

Polyak, B. T. Some methods of speeding up the convergence of iteration methods. \emph{USSR Computational Mathematics and Mathematical Physics}, 4(5):1–17, 1964.

Saleh, E. and Jiang, N. Deterministic bellman residual minimization. In \emph{Proceedings of Optimization Foundations for Reinforcement Learning Workshop at NeurIPS}, 2019.

Schoknecht, R. and Merke, A. TD(0) converges provably faster than the residual gradient algorithm. In \emph{Proceedings of International Conference on Machine Learning}, pp. 680–687, 2003.

Schulman, J., Moritz, P., Levine, S., Jordan, M., and Abbeel, P. High-dimensional continuous control using generalized advantage estimation. \emph{arXiv preprint arXiv:1506.02438}, 2015.

Sherman, J. and Morrison, W. J. Adjustment of an inverse matrix corresponding to a change in one element of a given matrix. \emph{Annals of Mathematical Statistics}, 21(1): 124–127, 1950.

Silver, D., Lever, G., Heess, N., Degris, T., Wierstra, D., and Riedmiller, M. Deterministic policy gradient algorithms. \emph{International Conference on Machine Learning}, pp. 387-395, 2014.

Strang, G. \emph{Linear algebra and its applications}. Thomson, Brooks/Cole, Belmont, CA, 2006.

Sun, W. and Bagnell, J. A. Online bellman residual algorithms with predictive error guarantees. 2015.

Sutton, R. S. Learning to predict by the methods of temporal differences. \emph{Machine Learning}, 3(1):9–44, 1988.

Sutton, R. S. and Barto, A. G. \emph{Reinforcement learning: An introduction}. MIT Press, 2018.

Sutton, R. S., Maei, H. R., Precup, D., Bhatnagar, S., Silver, D., Szepesv{\'a}ri, C., and Wiewiora, E. Fast gradient-descent methods for temporal-difference learning with linear function approximation. In \emph{Proceedings of the 26th International Conference on Machine Learning}, pp. 993–1000, 2009.

%Thomas, V. On the role of overparameterization in off-policy temporal difference learning with linear function approximation. In \emph{Advances in Neural Information Processing Systems}, 2022.

Tian, T. and Sutton, R. S. Extending sliding-step importance weighting from supervised learning to reinforcement learning. In \emph{Proceedings of the International Joint Conference on Artificial Intelligence}, pp. 67–82. Springer, 2019.

Tsitsiklis, J. and Van Roy, B. Analysis of temporal-difference learning with function approximation. In \emph{Advances in Neural Information Processing Systems}, volume 9, 1996.

Van Seijen, H., Van Hasselt, H., Whiteson, S., and Wiering, M. A theoretical and empirical analysis of expected Sarsa. In \emph{IEEE Symposium on Adaptive Dynamic Programming and Reinforcement Learning}, pp. 177–184, 2009.

Wang, Z. T. and Ueda, M. A convergent and efficient deep Q network algorithm. \emph{arXiv preprint arXiv:2106.15419}, 2021.

Watkins, C. J. and Dayan, P. Q-learning. \emph{Machine learning}, 8(3):279–292, 1992.

Yao, H., Bhatnagar, S., and Szepesv{\'a}ri, C. LMS-2: Towards an algorithm that is as cheap as LMS and almost as efficient as RLS. In \emph{Proceedings of the 48h Conference on Decision and Control (CDC)}, pp. 1181–1188. IEEE, 2009.

Zhang, J., He, T., Sra, S., and Jadbabaie, A. Why gradient clipping accelerates training: A theoretical justification for adaptivity. \emph{arXiv preprint arXiv:1905.11881}, 2019.

Zhang, S., Boehmer, W., and Whiteson, S. Deep residual reinforcement learning. In \emph{Proceedings of the 19th International Conference on Autonomous Agents and Multiagent Systems}, pp. 1611–1619, 2020.
 
Zhu, R. J. and Murray, J. M. Gradient descent temporal difference-difference learning. \emph{arXiv preprint  arXiv:2209.04624}, 2022.

\end{list}}

\fi

%%%%%%%%%%%%%%%%%%%%%%%%%%%%%%%%%%%%%%%%%%%%%%%%%%%%%%%%%%%%%%%%%%%%%%%%%%%%%%%
%%%%%%%%%%%%%%%%%%%%%%%%%%%%%%%%%%%%%%%%%%%%%%%%%%%%%%%%%%%%%%%%%%%%%%%%%%%%%%%
% APPENDIX
%%%%%%%%%%%%%%%%%%%%%%%%%%%%%%%%%%%%%%%%%%%%%%%%%%%%%%%%%%%%%%%%%%%%%%%%%%%%%%%
%%%%%%%%%%%%%%%%%%%%%%%%%%%%%%%%%%%%%%%%%%%%%%%%%%%%%%%%%%%%%%%%%%%%%%%%%%%%%%%
\newpage
\appendix
\onecolumn
\begin{center}
	\Large \bf Appendices
\end{center}

\section{Proof of  condition-number bounds}\label{app:proof cond}
In this appendix, we first present the proof of Theorem~\ref{th:cond} that involves bounds on the condition-number of $MSBE^V(\cdot)$.
We then establish similar bounds for $MSBE(\cdot)$ defined in \eqref{eq:MSBE}.

\subsection{Proof of  Theorem~\ref{th:cond}}\label{app:proof cond val}
Note that an MDP with a fixed policy boils down to a Markov chain with termination. 
Consider a Markov chain with termination that has $n$ non-terminal states, and let $P$ be its associated $n\times n$ transition matrix. 
Note that if transitions from a state can terminate with positive probability,  sum over the corresponding row of $P$ will be less than one.
Let
\begin{equation}\label{eq:def A}
A\defeq (I-\gamma P)^T(I-\gamma P).
\end{equation}
In the tabular setting and under uniform state distribution, we have $MSBE^V(\wb)=\wb^TA\wb/n$. Therefore, the condition-number $\cond$ of $MSBE^V(\cdot)$ is equal to the condition-number of $A$. 
Let $\lamax$  and $\lamin$ denote the largest and smallest eigenvalues of $A$, respectively.
It follows that
\begin{equation}\label{eq:eig ratio}
\cond = \frac{\lamax}{\lamin}.
\end{equation}

{\bf Proof of Part~(a).} 
We first propose an upper bound for $\lamin$ and then a lower bound for $\lamax$.
For states $i=1,\ldots,n$, let $l_i$ be the expected number of steps until termination when we start from state $i$ and follow the Markov chain's transitions.  
Then, for any state $i=1,\ldots,n$, we have
\begin{equation}\label{eq:Markov l}
l_i = 1+\sum_{j=1}^n P_{ij} l_j.
\end{equation}
Let $\lb \defeq \left[ l_1, \ldots, l_n  \right]^T$ 
%\begin{equation}
%\lb \defeq \left[ \begin{array}{c}l_1\\ \vdots\\ l_n \end{array} \right]
%\end{equation}
be the vector representation of $l_1,\ldots,l_n$.  %and $l\defeq (l_1+\cdots,l_n)/n$ be the mean of $l_1,\ldots,l_n$. 
Then, \eqref{eq:Markov l} can be written in the vector form as %$\lb = \oneb +P\lb$.
\begin{equation}\label{eq:lb 1}
\lb = \oneb +P\lb.
\end{equation}
where $\oneb$ is the vector of all ones. It then follows that 
\begin{equation}\label{eq:lb 2}
(I-\gamma P)\lb = \lb - \gamma P\lb = \lb - \gamma (\lb-\oneb) = (1-\gamma)  \lb +\gamma\oneb,
\end{equation}
where the second equality follows from \eqref{eq:lb 1}. 
Let $l\defeq (l_1+\cdots,l_n)/n$ be the mean of $l_1,\ldots,l_n$. 
Cauchy-Schwarz inequality implies that
\begin{equation}\label{eq:cauchy}
\frac{\|\lb\|^2}{n} =\frac{1}n \sum_{i=1}^n l_i^2  \ge \left(\frac{1}{n} \sum_{i=1}^n l_i\right)^2 =l^2.
\end{equation}

For the smallest eigenvalue of $A$, we have
\begin{equation} \label{eq:lamin}
\begin{split}
\lamin &\le \frac{\lb^TA\lb}{\|\lb\|^2}\\
&= \frac{\|(I-\gamma P)\lb\|^2}{\|\lb\|^2}\\
&= \frac{\|(1-\gamma )\lb + \gamma\oneb\|^2}{\|\lb\|^2}\\
&= \frac{(I-\gamma)^2\|\lb\|^2   +   2\gamma(1-\gamma) \oneb^T\lb  +\gamma^2 n }{\|\lb\|^2}\\
&= (I-\gamma)^2 +  \frac{2\gamma(1-\gamma) n l + \gamma^2 n }{\|\lb\|^2}\\
&\le (I-\gamma)^2 +  \frac{2\gamma(1-\gamma) l + \gamma^2 }{l^2}\\
& = \left(I-\gamma + \frac{\gamma}l \right)^2,
\end{split}
\end{equation}
where the first inequality follows from the definition of the smallest eigenvalue of a symmetric matrix, the first equality is due to the definition of $A$ in \eqref{eq:def A},
the second equality results from \eqref{eq:lb 2},
the fourth equality is from the definition of $l$, 
and the second inequality follows from \eqref{eq:cauchy}.

Let $\mbox{trace(A)}$ be the trace of $A$ defined as the sum of diagonal entries of $A$. 
It is well-known that the trace of any matrix is equal to the sum of eigenvalues of that matrix~\citer{(Strang, 2006)}\cite{strang2006LinearAlgebra}.
Therefore, for the largest eigenvalue of $A$, we have
\begin{equation}\label{eq:lamax}
\begin{split}
\lamax &\ge  \frac1n \mbox{trace}(A) \\
&= \frac1n \sum_{i=1}^n A_{ii}\\
&= \frac1n \sum_{i=1}^n \sum_{j=1}^n (I_{ji}-\gamma P_{ji})^2\\
&= \frac1n \sum_{i=1}^n \left((1-\gamma P_{ii})^2 + \sum_{j\ne i} \gamma^2 P_{ji}^2\right)\\
&\ge \frac1n \sum_{i=1}^n (1-\gamma P_{ii})^2 \\
&\ge \left(\frac1n \sum_{i=1}^n (1-\gamma P_{ii})\right)^2\\
&= \left(1- \frac{\gamma}n \sum_{i=1}^n  P_{ii}\right)^2\\
&= (1-\gamma h)^2,
\end{split}
\end{equation}
where the first inequality is because $\mbox{trace}(A)$ equals the sum of eigenvalues of $A$, 
the  first equality is from the definition of trace,
the second equality is due to the definition of $A$ in \eqref{eq:def A},
the third inequality follows from the Cauchy-Schwarz inequality, 
and the last equality is from the definition $h=\sum_{i=1}^n P_{ii}/n$ in the theorem statement.
Plugging \eqref{eq:lamin} and \eqref{eq:lamax} into \eqref{eq:eig ratio}, we obtain
\begin{equation*}
\cond = \frac{\lamax}{\lamin} 
\ge \frac{(1-\gamma h)^2}{\lamin}
\ge \frac{(1-\gamma h)^2}{\left(1-\gamma + {\gamma}/l \right)^2}
\ge \frac{(1-\gamma h)^2}{2(1-\gamma)^2 + 2\gamma^2/l^2 }
\ge \frac{(1-\gamma h)^2}{4}\, \min\left(\frac1{(1-\gamma)^2}, \frac{l^2}{\gamma^2} \right),
\end{equation*}
where the first and second inequalities are due to \eqref{eq:lamax} and \eqref{eq:lamin}, respectively.
This implies \eqref{eq:C1} and completes the proof of Part~(a) of Theorem~\ref{th:cond}.

{\bf Proof of Part~(b).} Consider an $n$-state Markov chain with transition matrix
\begin{equation}
P = \left[\begin{array}{cccc} 0&\cdots&0&1 \\ \vdots&\ddots&\vdots&\vdots  \\ 0&\cdots&0&1 \end{array}\right].
\end{equation}
In what follows, we derive  bounds on the largest and smallest eigenvalues of $A$ defined in \eqref{eq:def A}.
Let $\epsilon=\gamma/(n-1)$ and  $\vb=[-\epsilon,\ldots,-\epsilon,1]^T$. Then, $P\vb = \oneb$, and as a result,
\begin{equation}\label{eq:vb1}
\vb^TA\vb = \big\| (I-\gamma P)\vb \big\|^2
=\big\| \vb-\gamma \oneb \big\|^2
= (n-1)(\gamma+\epsilon)^2 + (1-\gamma)^2
\ge (n-1)(\gamma+\epsilon)^2
= \frac{n^2 \gamma^2}{n-1},
\end{equation}
where the first equality is from the definition of $A$ in \eqref{eq:def A}, and  the last equality is due to the definition of $\epsilon$. 
On the other hand, 
\begin{equation}\label{eq:vb2}
\| \vb\|^2 = (n-1)\epsilon^2+1 = \frac{\gamma^2}{n-1}+1 = \frac{n-1+\gamma^2}{n-1} \le \frac{n}{n-1}.
\end{equation}
It follows that
\begin{equation}\label{eq:lamaxb}
\lamax \ge \frac{\vb^TA\vb}{\|\vb\|}
\ge \frac{n^2\gamma^2/(n-1)}{n/(n-1)}
=n\gamma^2,
\end{equation}
where the second inequality is due to \eqref{eq:vb1} and \eqref{eq:vb2}.

In order to bound the smallest eigenvalue of $A$, let $\xb=[\gamma,\ldots,\gamma,1]^T$. Therefore $P\xb= \oneb$ and
\begin{equation}\label{eq:gam P x}
(I-\gamma P)\xb = \xb-\gamma P\xb = \xb-\gamma\oneb = [0,\ldots,0,1-\gamma]^T.
\end{equation}
It follows that 
\begin{equation}\label{eq:laminb}
\lamin \le \frac{\xb A \xb}{\|\xb\|^2} 
= \frac{\|(I-\gamma P)\xb\|^2}{\|\xb\|^2}
= \frac{(1-\gamma)^2}{\|\xb\|^2}
= \frac{(1-\gamma)^2}{(n-1)\gamma^2+1}
\le \frac{(1-\gamma)^2}{n\gamma^2}
\end{equation}
where the first equality is from the definition of $A$ in \eqref{eq:def A},
the second equality is due to \eqref{eq:gam P x},
and the third equality follows from the definition of $\xb$.
Plugging \eqref{eq:lamaxb} and \eqref{eq:laminb} into \eqref{eq:eig ratio}, we obtain
\begin{equation*}
\cond = \frac{\lamax}{\lamin} \ge \frac{n\gamma^2}{(1-\gamma)^2/(n\gamma^2)} = \frac{\gamma^4 n^2}{(1-\gamma)^2}.
\end{equation*}
This completes the proof of Part~(b) of Theorem~\ref{th:cond}.

\subsection{Condition-number of the MSBE defined in terms of action-values}\label{app:proof cond action-val}
Theorem~\ref{th:cond} involves bounds on the condition-number of $MSBE^V(\cdot)$ defined in \eqref{eq:MSBE v}. 
In this appendix, we establish similar bounds for $MSBE(\cdot)$ defined in \eqref{eq:MSBE}.

Given an MDP and a policy $\pi$,  we consider an \emph{induced augmented Markov chain} that is a Markov chain whose states are the state-action pairs of the MDP and its transition probabilities are as follows. 
For any $s,s'\in\mathcal{S}$ and $a,a'\in\mathcal{A}$, the probability of transition from $(s,a)$ to $(s',a')$ in the induced augmented Markov chain is
\begin{equation}\label{eq:aug mark}
p_\pi'(s',a'|s,a) \defeq \int_{r} p_\pi(s',a',r|s,a)\,dr
\end{equation}
where $p_\pi$ is defined in Section~\ref{sec:background}. 
We consider a tabular setting, and denote the condition-number of $MSBE_\distrib(\cdot)$ under uniform distribution $\distrib$ on state-action pairs by $\cond'$. We  let 
\begin{equation}
h' \defeq \frac1{nm} \sum_{s\in\mathcal{S}} \sum_{a\in\mathcal{A}}  p_\pi'(s,a|s,a)
\end{equation}
be the self-loop probability in the induced augmented Markov chain, where $n$ is the number of states and $m=|\mathcal{A}|$ is the number of actions.
Also let $l'$ be the expected number of steps until termination when starting from a uniformly random state-action pair.
The following proposition is the counterpart of Theorem~\ref{th:cond} for $\cond'$. % condition-number of $MSBE_{\textrm{}}(\cdot)$.
\begin{proposition}\label{prop:cond q}
In the tabular case, the  following statements hold  for any discount factor $\gamma\in[0,1]$:
\begin{itemize}
	\item[a)]
	For any MDP and any policy $\pi$,
	\begin{equation} \label{eq:C1 2}
	\cond' \ge \frac{(1-\gamma h')^2}{4} \, \min\left( \frac1{(1-\gamma)^2},\, l'^2\right).
	\end{equation}
	\item[b)]
	For any $n,m>0$, there exists an MDP with $n$ states and $m$ actions, and a policy $\pi$ for which, $\cond'\ge\gamma^4 (nm)^2/(1-\gamma)^2$.
\end{itemize}
\end{proposition}
\begin{proof}
	We can perceive the dynamics under any given MDP and policy as an induced augmented Markov chain defined in \eqref{eq:aug mark}.
	Applying the proof of Theorem~\ref{th:cond} on this induced augmented Markov chain implies Proposition~\ref{prop:cond q}.
\end{proof}

\medskip

\section{Experiment on condition number under linear function approximation} \label{app:cond large random matrix}
We ran an experiment to investigate the growth of condition number, $\cond$, in an extended Boyan chain under linear function approximation.
Fig.~\ref{fig:boyan2} shows the dependence of $\cond$ on the size of extended Boyan chain, under standard Boyan feature vectors  (see Appendix~\ref{app:empirical details boyan} for details). The number of standard Boyan chain features $d$ in this experiments, satisfies $n=4d-3$, where $n$ is the number of states.
We observe that the condition number can grow very large under linear function approximation even when $d/n<1$ (in this case $d/n\simeq 1/4$).
\begin{figure}[t!]
	\centering % the figures are generated by code_condition_number_large_random_matrix.py in my short codes
	\includegraphics[width=.5\linewidth]{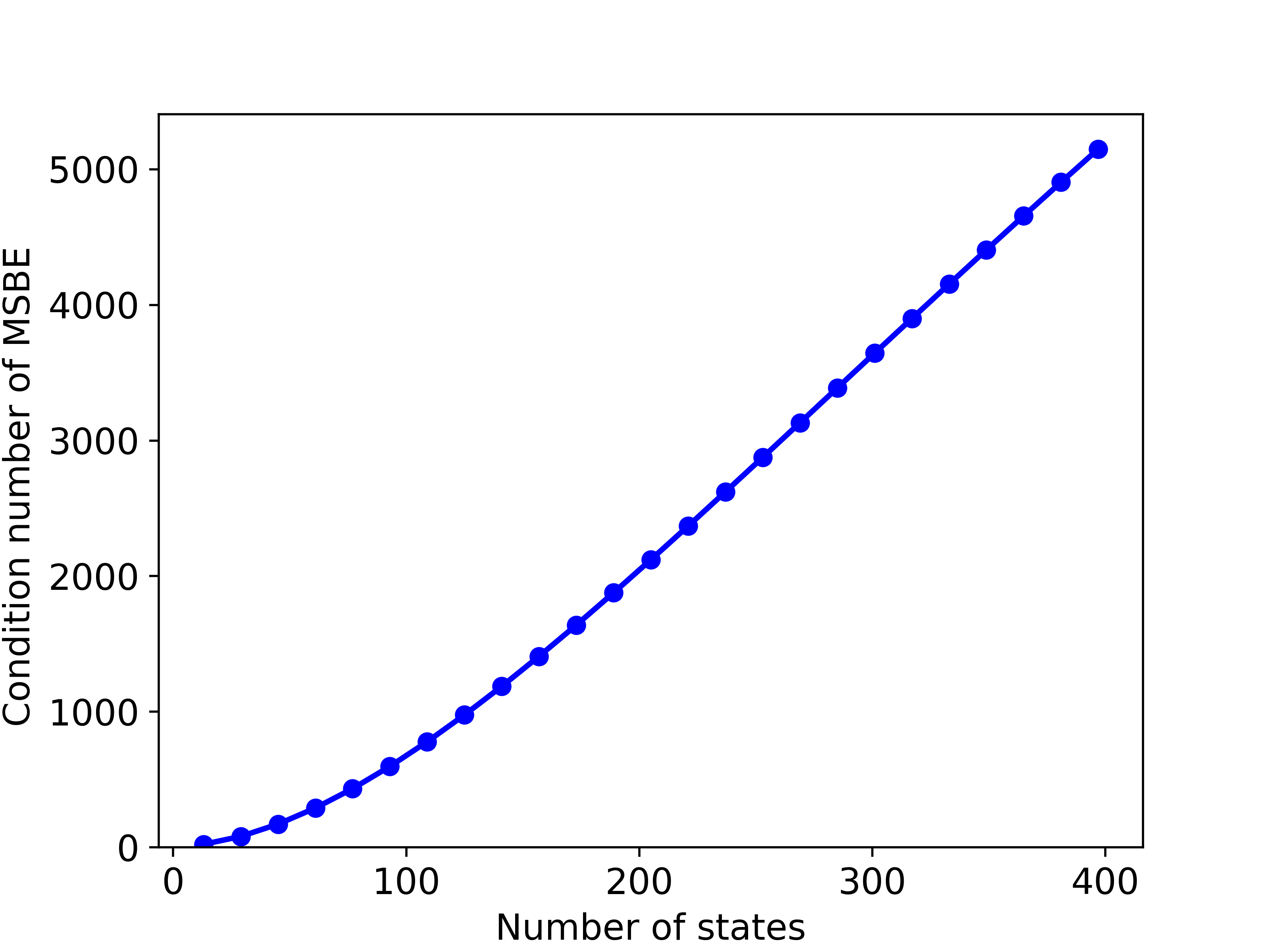}
	\caption{Condition-number of MSBE versus  number of states in an extended Boyan chain under linear function approximation with Boyan chain's standard features ($n=4d-3$). }\label{fig:boyan2}
\end{figure}

\section{ RAN as a proximal algorithm}\label{app: proximal view of RAN}
In this section, we provide further intuition for the RAN algorithm, by showing that Algorithm~\ref{alg:ran} can be equivalently derived as a proximal algorithm with momentum for minimizing MSBE.
Given an objective function $f$, a  proximal algorithm in its general form aims to find an approximate solution for a proximal operator of the following form in each iteration
\begin{equation}\label{eq:general proximal}
\wb_{t+1} \, \leftarrow\, \argmin{\wb} \Big( f(\wb) + D_t(\wb, \wb_{ref})  \Big),
\end{equation}
where $\wb_{ref}$ is a reference point, usually equal to $\wb_t$ or a trace of past $\wb$'s, and $D_t$ is  a penalty function (also called divergence) that encourages $\wb_{t+1}$ to stay close to $\wb_{ref}$.
In the special case that $D_t$ is a fixed Bregman divergence, \eqref{eq:general proximal} boils down to the mirror-descent algorithm \citer{(Juditsky \& Nemirovski, 2011)}\cite{JudiN11}.
However, in general, $D_t$ can be a time varying function and can depend on the local shape of the objective $f$. 

Back to the value estimation problem, for any consecutive state-action pairs $(s,a)$ and $(s',a')$, and any pair $\wb$ and $\wb'$ of weights, we let
\begin{equation}\label{eq:delta w sas'a'}
\delta_{\wb}(s,a,s',a') \defeq \gamma q_\wb(s',a')-q_\wb(s,a),
\end{equation} %\Exp_{s,a}\Big[\Exp_{s',a'|s,a}\big[\big(     \big)\big]\Big]
and 
\begin{equation}
\Delta\delta_{\wb,\wb'}(s,a,s',a') \defeq \delta_{\wb}(s,a,s',a') -\delta_{\wb'}(s,a,s',a').
\end{equation} 
Consider a divergence measure of the form 
$D_t(\wb,\wb_t)=c\, \Exp_{s,a}\Big[\Exp_{s',a'|s,a}\left[ \Delta\delta_{\wb,\wb_t}(s,a,s',a')^2  \right]\Big]$, for some constant $c>0$.
Then, the proximal operator in \eqref{eq:general proximal} turns into
\begin{equation} \label{eq:prox Delta delta}
\argmin{\wb}\, MSBE(\wb) + c\, \Exp_{s,a}\Big[\Exp_{s',a'|s,a}\left[ \Delta\delta_{\wb,\wb_t}(s,a,s',a')^2  \right]\Big].
\end{equation}

To obtain a low complexity incremental version of the above proximal updates, we consider doing sample updates along  the gradient of \eqref{eq:prox Delta delta}  at time $t$. 
For this purpose\footnote{Here we avoid using $\wb_{ref}=\wb_t$ because in this case, for any $(s,a,s',a')$,  $\Delta\delta_{\wb_t,\wb_t}(s,a,s',a')=0$. 
	This implies that $\nabla_\wb \Delta\delta_{\wb,\wb_t}(s,a,s',a')^2 \big|_{\wb=\wb_t}=0$. 
	As such,  the penalty would not affect gradient of the proximal objective at $\wb=\wb_t$.}, we work with $\wb_{ref}=\wb_{t-1}$ instead of $\wb_{ref}=\wb_{t}$,
i.e. we consider the proximal objective $MSBE(\wb) + c \,\Exp_{s,a}\Big[\Exp_{s',a'|s,a}\left[ \Delta\delta_{\wb,\wb_{t-1}}(s,a,s',a')^2  \right]$ and its unbiased sample gradient
\begin{equation}
\begin{split}
\gb_t &\defeq \big( \delta'_t - c\,\Delta\delta_{\wb_t,\wb_{t-1}}(S_t,A_t,S_{t+1}, A_{t+1})\big) \nabla\delta_t\\
&\simeq \big(\delta'_t - c\,(\wb_t-\wb_{t-1})^T\nabla\delta_t\big)\nabla\delta_t ,%+ O\big(\|\wb_t-\wb_{t-1}\|^2\big)
\end{split}
\end{equation}
where the approximate equality is because $(\wb_t-\wb_{t-1})^T\nabla\delta_t$ is a first order approximation of $\Delta\delta_{\wb_t,\wb_{t-1}}$.
Let $\hat{\gb}_t=\big(\delta'_t - c(\wb_t-\wb_{t-1})^T\nabla\delta_t\big)\nabla\delta_t$
and consider the approximate gradient update $\wb\leftarrow \wb-\beta\hat{\gb}_t$. 
We further employ a momentum to reduce the variance of these updates. The resulting momentum based algorithm is of the form
\begin{equation}\label{eq:prox incremental}
\begin{split}
\mb_t &= \lambda \mb_{t-1} + \beta\hat{\gb}_t,\\
\wb_{t+1} &= \wb_t - \alpha \mb_t.
\end{split}
\end{equation}
Let $\eta=1/\alpha$. Since $\wb_t-\wb_{t-1}=\alpha\mb_{t-1}$, $\hat{\gb}_t$ simplifies to 
%Note that $\wb_t-\wb_{t-1}=\alpha\mb_{t-1}$. Therefore, by letting $\eta=1/\alpha$, $\hat{\gb}_t$ simplifies to 
\begin{equation*}
\hat{\gb}_t = \big(\delta'_t - \mb_{t-1}^T\nabla\delta_t\big)\nabla\delta_t.
\end{equation*}
Plugging this into \eqref{eq:prox incremental}, we obtain updates that are identical to \eqref{eq:2TS}. 
This establishes our earlier claim that  the RAN algorithm can be equivalently formulated as a proximal algorithm with momentum for minimizing MSBE.

As another intuitive perspective for Algorithm~\ref{alg:ran}, we can perceive $\mb$ as a momentum of MSBE gradients, to which we add a correction term equal to a sample gradient of the penalty $\Exp_{s,a}\big[\Exp_{s',a'|s,a}\left[ \Delta\delta_{\wb,\wb_t}(s,a,s',a')^2  \right]\big]$, in each iteration. 
This momentum spreads the effect of each MSBE gradient over an $O(1/(1-\lambda))$-long horizon, which provides  enough time for the   correction updates to trim the direction of those MSBE gradients.
The correction terms are small along directions that $\nabla \delta$ is small, allowing  MSBE gradients to accumulate along those directions. 
Conversely, the correction terms are large along directions that $\nabla \delta$ is large, preventing $\mb$ from growing large along those directions. 
This leads to accelerated convergence along  the directions where $\nabla \delta$ is small, while preventing instability along directions where $\nabla \delta$ is large.

\section{Convergence of the RAN algorithm} \label{app:convergence proof}
In this appendix, we study conditions for convergence of Algorithm~\eqref{eq:2TS} under different choices of $\lambda$.
Convergence proofs of two-time-scale approaches are well-studied, and  generally involve tedious yet pretty  standard  statistical arguments.
As such, similar to several other papers \citer{(Konda \& Tsitsiklis, 1999; Dabney \& Thomas, 2014)}\cite{KondT99, DabnT14}, 
here we keep our arguments in a high level, and only discuss the steps of the proof  without going into the proof details. 
Throughout this Appendix, we consider irreducible and aperiodic Markov chains with finite number of states. We assume that the function approximation $Q_\wb(s,a)$ is a differentiable function of $\wb$ with bounded and Lipschitz constant derivatives. By boundedness of derivatives we mean that there exists a $C>0$ such that for consecutive state-action pairs $(s,a, s',a')$ and any $\wb$, 
\begin{equation}\label{eq:lipschitz}
	%\big\| \nabla Q_\wb(s,a) - \nabla Q_{\wb'}(s,a) \big\| \,\le\, C  \| \wb-\wb'\|,\qquad \mbox{and} \qquad  
	 \|\nabla \delta_{\wb}(s,a,s',a')\big\| < C,
\end{equation}
where $ \delta_{\wb}(s,a,s',a')$ is defined in \eqref{eq:delta w sas'a'}.
Given a distribution $\distrib$ over state and action pairs, for any $\wb\in\R^d$, we let
\begin{equation}\label{eq:def G hat}
\hat{G}_\wb \defeq \Exp_{s,a\sim\distrib}\Big[\Exp_{s',a'|s,a}\left[\nabla\delta_{\wb}(s,a,s',a') \, \nabla\delta_{\wb}(s,a,s',a')^T\right]\Big].
\end{equation}
We study convergence in three different regimes on $\lambda$, namely  $\lambda=1$, 
$\lambda=1-c\beta_t$ for some constant $c>0$,
and for a constant $\lambda<1$.

{\bf Case 1:} ($\lambda=1$). 
We assume that $\alpha_t$ and $\beta_t$ are decreasing positive sequneces, satisfying
\begin{equation}\label{eq:stepsizes of 2TS}
	\sum_{t=0}^\infty \alpha_t = \sum_{t=0}^\infty \beta_t = \infty, \qquad          \sum_{t=0}^\infty \alpha_t^2<\infty\,        \sum_{t=0}^\infty \beta_t^2 < \infty,      \qquad \frac{\alpha_t}{\beta_t}\to0.
\end{equation}
We further assume that $\hat{G}_\wb$ defined in \eqref{eq:def G hat}  is uniformly positive definite, in the sense that there is an $\epsilon>0$ such that for any $\wb$ and any $\xb\in\R^d$, we have 
$\xb^T\hat{G}_\wb \xb\ge \epsilon \|\xb\|^2$.
In this case, for fixed $\wb$, the updates in \eqref{eq:2TS} converge to $\hat{G}_\wb^{-1}\nabla MSBE_{\distrib}(\wb)$.
Since $\alpha_t/\beta_t\to0$,  $\wb$ is updated much slower than $\mb$. 
As such, $\mb$ is updated with much larger step-sizes, perceiving $\wb$ as almost stationary,  and therefore $\mb$ converges to an asymptotically small neighborhood of $\hat{G}_\wb^{-1}\nabla MSBE_{\distrib}(\wb)$. 
Since $\hat{G}_\wb$ is uniformly positive definite, this $\mb$ is an absolutely decreasing direction for $MSBE_{\distrib}(\wb)$. 
Then, standard proof techniques for stochastic approximation algorithms can be used to establish convergence of $\wb$ to a stationary point of MSBE.

{\bf Case 2:} ($\lambda=1-c\beta_t$, for some constant $c>0$).
In this case, the update rule \eqref{eq:2TS} boils down to:
\begin{equation}
	\begin{split}
	\mb\, &\leftarrow\lambda \mb + \beta_t (\delta'_t-\mb^T\nabla\delta_t\big)\nabla\delta_t\\
	&= (1-c\beta_t) \mb + \beta_t \big(\delta'_t-\mb^T\nabla\delta_t\big)\nabla\delta_t\\
	&=\mb - \beta_t \big((cI + \nabla\delta_t \nabla\delta_t^T)\mb   -\delta'_t  \nabla\delta_t \big)\\
	&=\mb - \beta_t \nabla_{\mb} \left(\frac12 \mb^T\big(cI + \nabla_\wb\delta_t \nabla_\wb\delta_t^T\big)\mb   -\delta'_t  \nabla_\wb\delta_t ^T\mb\right),
	\end{split}
\end{equation}
where $I$ is the identity matrix. 
Therefore, in each iteration, $\mb$ is updated along a sample gradient of the loss function $\hat{L}(\mb)=\mb^T \big(cI+ \hat{G}_\wb\big)\mb-\nabla_\wb MSBE(\wb)^T \mb$, where $\hat{G}_\wb$ is defined in \eqref{eq:def G hat}.  
Thus, assuming \eqref{eq:stepsizes of 2TS}, $\mb$ will asymptotically converge to the minimizer $\big(cI+ \hat{G}_\wb\big)^{-1}\nabla MSBE(\wb)$ of $\hat{L}$. 
Since $cI+ \hat{G}_\wb$ is uniformly positive definite, this is an absolutely descent direction for $MSBE(\wb)$. Then, standard proof techniques for stochastic approximation algorithms can be used to establish convergence of $\wb$ to a stationary point of MSBE.

{\bf Case 3: }  (Constant $\lambda<1$).
As opposed to the Cases~1 and~2, here we do not need two-time-scale step-sizes. More specifically, we assume that $\alpha_t>0$ is constant and $\beta_t$ is a decreasing positive sequence satisfying $\sum_{t=0}^\infty \beta_t = \infty$  and  $\sum_{t=0}^\infty \beta_t^2<\infty $. 
We also assume  that $\wb$ remains bounded which can be enforced either by projection of $\wb$ on a compact set or by adequate normalization of $\alpha_t$ (see \citer{(Konda \& Tsitsiklis, 1999)}\cite{KondT99} and \citer{(Kushner \& Yin, 2003)}\cite{KushY03}).
%Let $C'=\max_{(s,a,s',a')}|\delta_{\bf 0}(s,a,s',a')|$. 
It follows from \eqref{eq:lipschitz} that for any time $t$,
\begin{equation}\label{eq:bound delta t}
	 |\delta_t| = |\delta_{\wb_t}(S_t,A_t,S_{t+1},A_{t+1})|  %\le \delta_{\bf 0}(S_t,A_t,S_{t+1},A_{t+1}) + C\|\wb_t\| 
	 \le C'+C\|\wb_t\| .
\end{equation}
for some constant $C'$. Then,  from \eqref{eq:2TS} we obtain
\begin{equation}
	\begin{split}
		\|\mb_t\| &= \big\|\lambda  \mb_{t-1}+ \beta_t\big(\delta_t'-\mb_{t-1}^T\nabla\delta_t\big)\nabla \delta_t  \big\|\\
		&\le \lambda \|\mb_{t-1}\| +\beta_t \big( |\delta_t'|  + \|\mb_{t-1}\|\cdot \|\nabla\delta_t\|\big) \|\nabla\delta_t\|\\
		&\le \lambda \|\mb_{t-1}\| +\beta_t \big( |\delta_t'|  + C \|\mb_{t-1}\|\big) C\\
		&\le\lambda \|\mb_{t-1}\| +\beta_t C \big( C'+C\|\wb_t\| + C \|\mb_{t-1}\|\big)  \\
		&= (\lambda +\beta_t C^2) \|\mb_{t-1}\| + \beta_t ( C'C+C^2\|\wb_t\|),
	\end{split}
\end{equation}
where the second and third inequalities follow from \eqref{eq:lipschitz} nad \eqref{eq:bound delta t}, respectively. 
When $\beta_t$ is small enough such that $\lambda +\beta_t C^2<1$, it follows from the boundedness assumption of $\wb$ that $\|\mb_t\| = O(\beta_t/(1-\lambda))$.
Therefore, \eqref{eq:2TS} can be expressed as
\begin{equation*}
\begin{split}
\mb_t &=  \lambda  \mb_{t-1}+ \beta_t \delta_t'\nabla \delta_t-\beta_t (\mb_{t-1}^T\nabla\delta_t) \nabla \delta_t	\\
&=  \lambda  \mb_{t-1}+ \beta_t \delta_t'\nabla \delta_t + O\big(\beta_t^2/(1-\lambda)\big) \\
&=  \sum_{\tau=0}^t \big(\lambda^\tau \beta_{t-\tau} \delta_{t-\tau}'\nabla \delta_{t-\tau} \big)+  O\big(\beta_t^2/(1-\lambda)^2\big).
\end{split}
\end{equation*}
Therefore, $\mb$ is essentially a momentum of sample gradient of MSBE.
Then, convergence of $\wb$ to a stationary point of MSBE follows from standard techniques for analysing SGD algorithms with momentum.

%%%%%%%%%%%%%%%%%%%%%%%%%%%%%%%%
\medskip
\section{Asymptotic robustness of RAN to parameterization} \label{app:invariance}
In this appendix, we present a formal version of Proposition~\ref{prop:invariance} and   its proof. 
Here, we assume that $\lambda=1$ and consider an asymptotically small step-sizes regime with $\alpha\to0$ and $\alpha/\beta\to0$.
As discussed in Section~\ref{sec:alg1} and Appendix~\ref{app:convergence proof}, when $\alpha/\beta\to0$, $\mb$ in the RAN algorithm converges to the approximate Gauss-Newton direction
\begin{equation} \label{eq:def mGN}
\mb_{GN}(\wb,q) = \hat{G}_{\wb,q}^{-1} \,\,g_{\wb,q},
\end{equation}
where  
\begin{equation}\label{eq:Gwq}
\hat{G}_{\wb,q} =  \Exp_{s,a\sim\distrib}\Big[\Exp_{s',a'|s,a}\left[ \big(\gamma \nabla_\wb q_\wb(s',a') - \nabla_\wb q_\wb(s,a)\big)\, \big(\gamma \nabla_\wb q_\wb(s',a') - \nabla_\wb q_\wb(s,a)\big)^T \right]\Big],
\end{equation}
\begin{equation}\label{eq:gwq}
g_{\wb,q} = \Exp_{s,a\sim\distrib}\Big[\Exp_{s',a',r|s,a}\left[r+ \gamma q_\wb(s',a') - q_\wb(s,a)  \right]\,   \Exp_{s',a'|s,a} \big[\gamma \nabla_\wb q_\wb(s',a') - \nabla_\wb q_\wb(s,a)\big]\Big].
\end{equation}
where $\distrib$ is a distribution over state and action pairs. In this case, when $\alpha\to0$ the trajectory of $\wb$ approaches the trajectory of the following  Ordinary Differential Equations (ODE):
\begin{equation} \label{eq:ode}
\dot{\wb} = -\mb_{GN}(\wb,q),
\end{equation}
where $\dot{\wb}$ is the derivative of $\wb$ with respect to time.
We refer to \eqref{eq:ode} as the ODE formulation of two time-scale RAN for the $q$ function.

Let $u:\R^d\to\R^d$ be a (possibly non-linear) bijective and differentiable mapping. Suppose that the Jacobian of $u(\cdot)$ is invertible everywhere.
We define $\tilde{q}$ as a reparameterization of the $q$ function by $u$. More specifically, for any state-action pair $(s,a)$ and any $\vb\in\R^d$, we let
\begin{equation}\label{eq:qtilde}
\tilde{q}_\vb(s,a) = q_{u(\vb)}(s,a).
\end{equation}
Consider the ODE formulation of two time-scale RAN for the $\tilde{q}$ function:
\begin{equation} \label{eq:ode2}
\dot{\vb} = -\mb_{GN}(\vb,\tilde{q}),
\end{equation}
for $\mb_{GN}(\wb,\tilde{q})$ defined in \eqref{eq:def mGN}.

The following proposition is a formal version of Proposition~\ref{prop:invariance}. It draws a connection between  solution trajectories of \eqref{eq:ode} and \eqref{eq:ode2}. 
\begin{manualProposition}{6.1}[Formal] \label{prop:formal}
	Let $u:\R^d\to\R^d$ be a (possibly non-linear) bijective and differentiable mapping with invertible Jacobian, and consider the corresponding reparameterization $\tilde{q}$ of the $q$ function as in \eqref{eq:qtilde}. Let $\wb_t$ and $\vb_t$ for $t\ge0$ be solution trajectories of  the ODE formulations of two time-scale RAN in \eqref{eq:ode} and \eqref{eq:ode2}, respectively. 
	If $\wb_0=u\big(\vb_0\big)$, then $\wb_t=u(\vb_t)$, for all $t\ge0$.
	Moreover, $q_{\wb_t}(s,a) = \tilde{q}_{\vb_t}(s,a)$, for all times $t\ge0$ and all state-action pairs $(s,a)$.
\end{manualProposition}
The proposition suggests that the two-time scale RAN algorithm with asymptotically small step-sizes is invariant to any non-linear bijective reparameterization of the $q$ function.
\begin{proof}[Proof of Proposition~\ref{prop:formal}]
	For any $\vb\in\R^d$, let $J_\vb$ be the Jacobian matrix of $u(\vb)$. Then, for any state-action pair $(s,a)$ and any $\vb\in\R^d$, 
	\begin{equation}\label{eq:chain rule qhat}
	\nabla_\vb \tilde{q}_{\vb}(s,a) = \nabla_\vb q_{u(\vb)}(s,a) = J_\vb^T \,\nabla_\wb q_{\wb} (s,a) \big|_{\wb=u(\vb)},
	\end{equation}
	where the first equality is from the definition of $\tilde{q}$ in \eqref{eq:qtilde}, and the second equality follows from the chain rule for differentiation.
	Then, %the definition of $\hat{G}$ in \eqref{eq:Gwq}, we have
	\begin{equation}\label{eq:GhatJ}
	\begin{split}
	\hat{G}_{\vb,\tilde{q}} &=  \Exp_{s,a\sim\distrib}\Big[\Exp_{s',a'|s,a}\left[ \big(\gamma \nabla_\vb \tilde{q}_\vb(s',a') - \nabla_\vb \tilde{q}_\vb(s,a)\big)\, \big(\gamma \nabla_\vb \tilde{q}_\vb(s',a') - \nabla_\vb \tilde{q}_\vb(s,a)\big)^T \right]\Big]\\
	&=  \Exp_{s,a\sim\distrib}\Big[\Exp_{s',a'|s,a}\left[ J_\vb^T \big(\gamma \nabla_\wb q_\wb(s',a') - \nabla_\wb q_\wb(s,a)\big)\, \big(\gamma \nabla_\wb q_\wb(s',a') - \nabla_\wb q_\wb(s,a)\big)^T J_\vb \,\big|_{\wb=u(\vb)} \right]\Big]\\
	&=J^T_\vb \hat{G}_{\wb,q} \,J_\vb\,\,\big|_{\wb=u(\vb)},
	\end{split}
	\end{equation}
	where the first and the last equalities are due to \eqref{eq:Gwq} second equality follows from \eqref{eq:chain rule qhat}
	In the same vein, 
	\begin{equation}\label{eq:ghatJ}
	\begin{split}
	g_{\vb,\tilde{q}} &= \Exp_{s,a\sim\distrib}\Big[\Exp_{s',a',r|s,a}\left[r+ \gamma \tilde{q}_\vb(s',a') - \tilde{q}_\vb(s,a)  \right]\,   \Exp_{s',a'|s,a} \big[\gamma \nabla_\vb \tilde{q}_\vb(s',a') - \nabla_\vb \tilde{q}_\vb(s,a)\big]\Big]\\
	&= \Exp_{s,a\sim\distrib}\Big[\Exp_{s',a',r|s,a}\left[r+ \gamma \tilde{q}_\vb(s',a') - \tilde{q}_\vb(s,a)  \right]\,   \Exp_{s',a'|s,a} \big[\gamma J_\vb^T \nabla_\wb q_\wb(s',a') - J_\vb^T \nabla_\wb q_\wb(s,a)\big] \,\,\big|_{\wb=u(\vb)}\Big]\\
	&= \Exp_{s,a\sim\distrib}\Big[\Exp_{s',a',r|s,a}\left[r+ \gamma q_\wb(s',a') - q_\wb(s,a)  \right]\,   \Exp_{s',a'|s,a} \big[\gamma J_\vb^T \nabla_\wb q_\wb(s',a') - J_\vb^T \nabla_\wb q_\wb(s,a)\big] \,\,\big|_{\wb=u(\vb)}\Big]\\
	&=J_\vb^T g_{\wb,q} \,\,\big|_{\wb=u(\vb)},
	\end{split}
	\end{equation}
	where the first and the last equalities are due to \eqref{eq:gwq}, the  second equality follows from \eqref{eq:chain rule qhat}, and the third equality is from the definition of $\tilde{q}$ in \eqref{eq:qtilde}.
	Plugging \eqref{eq:GhatJ} and \eqref{eq:ghatJ} into \eqref{eq:ode} and \eqref{eq:ode2}, we obtain
	\begin{equation}
	\begin{split}
	 \dot{\vb} &= -\mb_{GN}(\vb,\tilde{q}) \\
	 &= -\hat{G}_{\vb,\tilde{q}}^{-1} \,\,g_{\vb,\tilde{q}} \\
	 &= -\big(J_\vb^T \,\hat{G}_{\wb,q} \,J\big)^{-1} \,\,g_{\vb,\tilde{q}} \,\,\big|_{\wb=u(\vb)}\\
	 &= -J^{-1}_\vb\,\hat{G}_{\wb,q} ^{-1}J_\vb^{-T}\, g_{\vb,\tilde{q}} \,\,\big|_{\wb=u(\vb)}\\
	 &= -J^{-1}_\vb\,\hat{G}_{\wb,q}^{-1} \,\,g_{\wb,q}  \,\,\big|_{\wb=u(\vb)}\\
	 &=-J^{-1}_\vb\,\mb_{GN}(\wb,q) \,\,\big|_{\wb=u(\vb)}\\
	 &= J^{-1}_\vb\,\dot{\wb} \,\,\big|_{\wb=u(\vb)},
	 \end{split}
	\end{equation}
	where the equalities are respectively due to    \eqref{eq:ode2}, 
	\eqref{eq:def mGN}, 
	\eqref{eq:GhatJ}, 
	the assumption that the Jacobian $J$ of $u$ is invertible,
	\eqref{eq:ghatJ},
	\eqref{eq:def mGN}, 
	and \eqref{eq:ode}. 
	Therefore, if at time $t$, $\wb_t=u(\vb_t)$, then $\dot{\wb_t} = J_\vb\dot{\vb_t} = \frac{d}{dt}u(\vb_t)$. 
	This together with the assumption $\wb_0=u(\vb_0)$ implies that $\wb_t=u\big(\vb_t\big)$, for all $t\ge0$.
	It then follows from the definition of $\tilde{q}$ in \eqref{eq:qtilde} that  $q_{\wb_t}(s,a) = \tilde{q}_{\vb_t}(s,a)$, for all times $t\ge0$ and all state-action pairs $(s,a)$.
	This completes the proof of Proposition~1.
\end{proof}

%%%%%%%%%%%%%%%%%%%%%%%%%%%%%%%%

\medskip
\section{Pseudo code of RANS} \label{app:rans}
The pseudo code of the RANS algorithm is given in Algorithm~\ref{alg:rangs}. 
Note that in the two time scale regime, where $\alpha, \eta\to0$ and $\alpha/\eta\to0$, outlier-splitting would have no effect on the expected direction of $\mb$ updates. 
In this case, convergence of the RANS algorithm follows from  similar arguments to Appendix~\ref{app:convergence proof}.

\begin{algorithm}[tb]
	\caption{\quad RANS}
	\label{alg:rangs}
	\begin{algorithmic}
		\STATE {\bfseries Hyper parameters:} stepsize $\alpha$, $\eta\in(0,1)$ outlier threshold $\rho>1$, momentum and trace parameters $\lambda,\lambda'\in[0,1)$, outlier sampling probability $\sigma$. Good default values for $\eta$, $\rho$,  $\lambda$, $\lambda'$, and $p$ based on our experiments are $\eta=0.2$, $\rho=1.2$,  $\lambda=0.999$, $\lambda'=0.9999$, and $\sigma=0.02$ .
		\STATE {\bfseries Initialize:} $\mb=\mathbf{0}$, $\hat{\nub}=\mathbf{0}$, $\hat{\xi}=0$, and $\wb$.
		\FOR{$t=1,2,\ldots$}
		\STATE Take two independent samples $(S_{t+1},A_{t+1})$ and $(S_{t+1}',A_{t+1}')$, and consider $\delta_t$ and $\delta'_t$ as in \eqref{eq:delta t} and \eqref{eq:delta' t}.
		\STATE $\hat{\nub}_t\leftarrow \lambda'\hat{\nub}_{t-1}+(1-\lambda')(\nabla\delta_t)^2$    \hspace{1.1cm} \algcomment{$(\nabla\delta_t)^2$ is the  entrywise square vector of $\nabla\delta_t$}
		\STATE $\nub_t =  \hat{\nub}_t / (1-\lambda'^t)$     \hspace{3cm} \algcomment{bias-corrected trace of $(\nabla\delta_t)^2$}
		\STATE $\xi_t = \langle(1/{\sqrt{\nub_t}})\odot \nabla\delta_t, \nabla\delta_t\rangle$   \hspace{1.5cm} \algcomment{Outlier measure}
		\STATE $\hat{\xi}\leftarrow  \lambda'\hat{\xi}+(1-\lambda')\xi_t$
		\STATE $\bar\xi =  \hat{\xi} / (1-\lambda_\xi'^t)$   \hspace{3.3cm} \algcomment{bias-corrected trace of $\xi$}
		\STATE $k=\lfloor \xi_t/(\rho\bar{\xi}) \rfloor +1$   \hspace{2.9cm} \algcomment{outlier-splitting factor}
		\STATE $\betab={\eta}/(\rho \bar{\xi}_t \sqrt{\nub_t})$
		\STATE $\mb \leftarrow  \lambda  \mb+ \big(\delta_t'-\mb^T\nabla\delta_t\big)\betab\odot\nabla \delta_t/k$
		\STATE $\wb\leftarrow \wb-\alpha\mb$
		%\STATE {\bfseries if}{ $k>1$}{\bfseries:} store  $(f,k,k-1)$ in the outlier buffer
		\IF{$k>1$}
		\STATE Store  $(S_t,A_t,S_{t+1}, A_{t+1}, r_t,,S_{t+1}', A_{t+1}', r_t', k,k-1)$ in the outlier buffer   \algcomment{the last entry in the\\ \hfill  tuple indicates the remaining number of future updates based on this sample}
		\ENDIF
		\STATE {\bfseries With probability}  $\min(1, \sigma*\mbox{length of outlier bufffer})${\bfseries:}    \algcomment{do an update using outlier buffer samples}
		\bindent
		\STATE Sample $(S_\tau,A_\tau,S_{\tau+1}, A_{\tau+1}, r_\tau,S_{\tau+1}', A_{\tau+1}', r_\tau', k',j)$ uniformly form the outlier buffer
		\STATE Let $\delta_t$ and $\delta'_t$ be as in \eqref{eq:delta t} and \eqref{eq:delta' t}, respectively
		\STATE $\xi = \langle(1/{\sqrt{\nub_t}})\odot \nabla\delta_\tau, \nabla\delta_\tau\rangle$
		\STATE  $k''=\max\big(k', \lfloor \xi/(\rho\bar{\xi}) \rfloor +1\big)$
		\STATE $\mb \leftarrow  \lambda  \mb+ \big(\delta_\tau'-\mb^T\nabla\delta_\tau'\big)\betab\odot\nabla \delta_\tau/k''$
		\STATE $\wb\leftarrow \wb-\alpha\mb$
		\IF{$j>1$}
		\bindentb
		\STATE Replace $(S_\tau,\ldots, k',j)$ with $(S_\tau,\ldots, k',j-1)$ in the  outlier buffer
		\eindent
		\ELSIF{$j=1$}
		\bindentb
		\STATE Remove $(S_\tau,\ldots, k',j)$ from the outlier buffer
		\eindent
		\ENDIF
		%\STATE  {\bfseries if}{ $j>1$}{\bfseries:} Replace $(f,k',j)$ with $(f,k',j-1)$ in the oultiler buffer
		%\STATE  {\bfseries if}{ $j=1$}{\bfseries:} Remove $(f,k',j)$ from the  buffer
		\eindent
		\ENDFOR
	\end{algorithmic}
\end{algorithm}

%%%%%%%%%%%%%%%%%%%%%%%%%%%%%%%%5
\medskip
\section{Details of experiments}
\subsection{Experiments of Section~\ref{sec:cond} and Appendix~\ref{app:cond large random matrix}} \label{app:empirical details boyan}
An  $n$-state extended Boyan chain with termination is a Markov chain with termination with states $0,1,\ldots,n-1$ and transition probabilities: $P(1\to0)=1$ and $P(i\to i-1)=P(i\to i-2)=0.5$ for $i=2,3,\ldots,n-1$. Furthermore, state $0$ goes to a terminal state with probability $1$.  
By standard features, we mean feature representations similar to \citer{(Boyan, 2002)}\cite{Boya02}. % (****** need fancy reference ******) 
More specifically, given a $d>1$ and $n=4d-3$, we consider $d$ standard features for the $n$-state  extended Boyan-chain  as follows. In the $i$th standard feature vector for $i=0,1,\ldots,d-1$, the $j$th entry is equal to $1-|j-4i|/4$ for $j=\max(0,4i-3),\ldots, \min(d-1,4i+3)$, and all other entries equal zero.
For the special case of $n=13$ and $d=4$, the standard features would be the same as the features considered in \citer{(Boyan, 2002)}\cite{Boya02}.

By random binary features we mean an $n\times d$ feature matrix $\Phi$ with i.i.d. entries that take $0$ and $1$ values with equal probability.  
For evaluating value-errors in Fig.~\ref{fig:boyan} we consider reward $1$ for the transition at state $0$ (that leads to the terminal state), and reward $0$ for all other transitions. 
Note that the reward function does not affect  condition-number. 
Each point in this first experiment is the median of $100$ independent runs with different random feature matrices. 
We use median instead of mean to eliminate the adverse  effect of unbounded values in degenerate cases (e.g., infinite condition-number in the case of low rank feature matrices).

For both experiments in Fig.~\ref{fig:boyan} and~\ref{fig:boyan2} we use discount factor $\gamma=0.995$.   
Condition-numbers and value-errors in these experiments are computed with respect to uniform state distribution.  % this is not much different from online distribution that starts from state $n-1$, because that distribution is very close to uniform.

\subsection{Experiment of Section~\ref{sec:alg1}} \label{app:exp Baird6} 
We considered  an extension of the Hallway environment with $50$ states and one action, in which each state $s=1,2\ldots,50$ transits to state $\min(s+1,50)$ with probability $0.99$, and transits to a terminal state with probability $0.01$.
The experiment was tabular with $\gamma=1$.
All rewards were set equal to zero, in which case the correct $q$-values are $q_\pi(s,\cdot)=0$, for all sates $s$.
All algorithms were initialized with $q(s,\cdot)=1$, for $s=1,\ldots,50$.
Fig.~\ref{fig:Baird50} illustrates learning curves of value-error $\sum_{s=1}^{50}\big(q(s,1)-q_\pi(s,1)\big)^2/50$ for RAN (Algorithm~\ref{alg:ran}), RG, and TD(0) algorithms.
Each point  is an average over 100 independent runs. 
We optimized the parameters for RG and Algorithm~\ref{alg:ran}.
The parameters used in the experiments are as follows. For RAN, we set $\alpha=0.025$, $\beta=0.4$, and $\lambda=0.9998$.
For RG and TD(0) we used $\alpha=0.5$.

\subsection{Experiment of Section~\ref{sec:GTD}} \label{app:exp Baird star}
We ran an experiment on Baird's Star environment~\citer{(Baird, 1995)}\cite{Bair95}. 
We performed off-policy learning with uniform state distribution.
All rewards were set to zero, in which case the correct $q$-values are $q_\pi(s,\cdot)=0$, for all sates $s$.
We used the initial point $\wb=[2,1,1,1,1,1,1]$ for all algorithms.
Fig.~\ref{fig:Baird star} illustrates  learning curves of value-error $\sum_{s=1}^{6}\big(q_\wb(s,1)-q_\pi(s,1)\big)^2/6$ for the RAN (Algorithm~\ref{alg:ran}), DSF-RAN (Algorithm~\ref{alg:dsf-rang}), RG, GTD2, and TD(0) algorithms.
Each point is an average over 10 independent runs.
TD(0) was unstable on this environment, and we chose a very small step-size $\alpha=10^{-5}$ for TD(0). 
For all other algorithms, we used optimized parameters as follows.
For RAN, we set $\alpha=2$, $\beta=0.15$, and $\lambda=0.995$.
For DSF-RAN, we set $\alpha=1$, $\beta=0.15$, $\eta=0.3$, and $\lambda=0.995$. 
For RG we used $\alpha=0.3$.
For GTD2 we set $\alpha=0.15$ and $\beta=0.3$.

\subsection{Experiment of Section~\ref{sec:experiments}} \label{app:exp control}
We ran an experiment on classic control tasks --Acrobot and Cartpole-- to test the performance of the RANS algorithm.
We used a single-layer neural network with 64 hidden ReLU units to learn the action-values, while choosing actions according to a softmax distribution on the action-values, $a\sim \textrm{Softmax} \big(q_\wb(s,\cdot)\big)$.
The network for $q_\wb$ was trained with three algorithms: TD(0), RG, and RANS (Algorithm~\ref{alg:rangs}).
Since RANS incorporates adaptive step-sizes, for fair comparison we trained TD(0) and RG using Adam optimizer.

For each algorithm, we performed training for 100 randomly generated random seeds.
For each training seed, 
in order to obtain an estimate of expected returns, we took an average over 400 independent environment simulations once every 500 steps. 
%every 500 steps we averaged the returns of 400 independent environment simulations  to obtain an estimate of the expected return for that training seed. 
In Fig.~\ref{fig:control mean} we plot the average of these estimated expected returns over  the 100  training seeds. 

In the Cartpole environment, once an algorithm reaches score 500, it will not see any failure for many episodes in a row. In this case, the agent starts to forget the actions that prevented failure and led it to success, causing the performance to drop before it can  rise again. 
This phenomenon is known as catastrophic forgetting, and leads to large oscillations in learning curves.
In order to hide the effect of catastrophic forgetting on the learning curves, in Fig.~\ref{fig:control} we eliminated the 50 worst return estimates (corresponding to the 50 worst training seeds) at each point and plotted the mean of top 50 return estimates. The shades in Fig.~\ref{fig:control}  and Fig.~\ref{fig:control mean}  show the 99 percent confidence intervals over the averaged data.

We did not use replay buffer and batch updates. We used a small quadratic regularizer with coefficient $10^{-5}$ on the weights of the neural network.
For all algorithms, we used optimized parameters that maximize area under the curves, as follows.
In Acrobot experiment:
\begin{itemize}
\item
For TD(0), we used softmax coefficient $1$ and  Adam optimizer with step-size $0.005$.
\item
For RG, we used softmax coefficient $16$ and  Adam optimizer with step-size $0.001$.
\item
For RANS, we used softmax coefficient $16$ and  $\alpha=0.005$  and set all other parameters were set to their default values  described in Algorithm~\ref{alg:rangs}.
\end{itemize}
In Cartpole experiment:
\begin{itemize}
	\item
	For TD(0), we used softmax coefficient $0.005$ and  Adam optimizer with step-size $0.3$.
	\item
	For RG, we used softmax coefficient $0.002$ and  Adam optimizer with step-size $0.3$.
	\item
	For RANS, we used softmax coefficient $8$ and  $\alpha=0.001$  and set all other parameters were set to their default values  described in Algorithm~\ref{alg:rangs}.
\end{itemize}

	In another experiment, we evaluated the performance of RANS on simple MuJoCo environments--Hopper and HalfCheetah.
	We used the standard (300$\times$400) two layer DDPG neural network architecture \cite{Lillicrap2015DDPG}\citer{(Lillicrap et al., 2015)} for the actor and the Q-network. 
	The network for $q_\wb$ was trained with three algorithms: TD with delayed targets, RG, and RANS (Algorithm~\ref{alg:rangs}).
	For TD and RG we used Adam optimizer.
	The actor was trained using a deterministic policy gradient, in principle similar to  \cite{Lillicrap2015DDPG}\citer{(Lillicrap et al., 2015)} but with Adam optimizer and without delayed target actor.
	
	For each algorithm, we performed training for 100 randomly generated random seeds.
	In Fig.~\ref{fig:MuJoCo mean} we plot average returns over  the 100  training seeds. 
	
	We used no replay buffer and no batch updates in our experiments. We employed a small quadratic regularizer with coefficient $10^{-5}$ on the weights of the neural network for $q_\wb$ .
	For all algorithms, we optimized all parameters to maximize area under the return curves. Parameters used for the Hopper environment are as follows.
	\begin{itemize}
		\item
		For TD, we used delayed target with Polyak averaging factor $0.001$ and  Adam optimizer with step-size $0.005$ for Q-network updates. We set $\gamma=0.99$ and used actor learning rate $10^{-8}$.
		\item
		For RG, we used Adam optimizer with step-size $0.001$ for Q-network updates, actor learning rate $10^{-8}$, and $\gamma=0.999$.
		\item
		For RANS, we set $\alpha=0.001$ and all other parameters were set to their default values (see Algorithm~\ref{alg:rangs}). In this experiment, we sued  actor learning rate $10^{-7}$ and  $\gamma=0.999$.
	\end{itemize}
	Parameters used for the HalfCheetah environment are as follows:
	\begin{itemize}
		\item
		For TD, we used delayed target with Polyak averaging factor $4\times 10^{-5}$ and  Adam optimizer with step-size $3\times 10^{-5}$ for Q-network updates. We set $\gamma=0.99$ and used actor learning rate $10^{-5}$.
		\item
		For RG, we used Adam optimizer with step-size $10^{-5}$ for Q-network updates, actor learning rate $10^{-5}$, and $\gamma=0.999$.
		\item
		For RANS, we set $\alpha=0.003$ and all other parameters were set to their default values (see Algorithm~\ref{alg:rangs}). In this experiment, we sued  actor learning rate $10^{-6}$ and  $\gamma=0.999$.
	\end{itemize}

\begin{figure}[t!]
	\centering % the figures are generated by code_condition_number_large_random_matrix.py in my short codes
	\begin{subfigure}[t]{0.4\linewidth}
		\centering
		\includegraphics[width=1\linewidth]{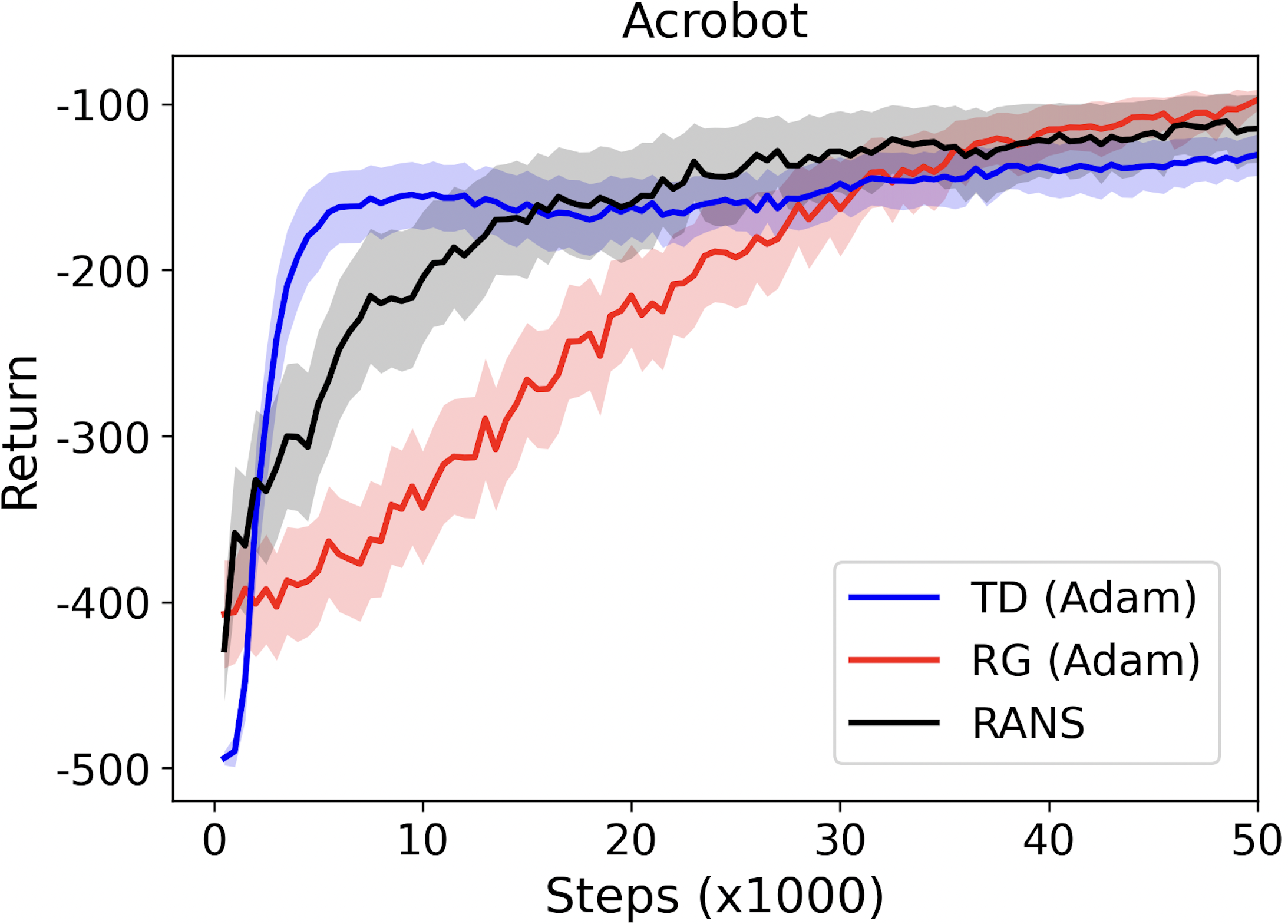}
	\end{subfigure}%
	%\\
	%\vspace{.3cm}
	\hspace{1.2cm}
	\begin{subfigure}[t]{.4\linewidth}
		\centering
		\includegraphics[width=1\linewidth]{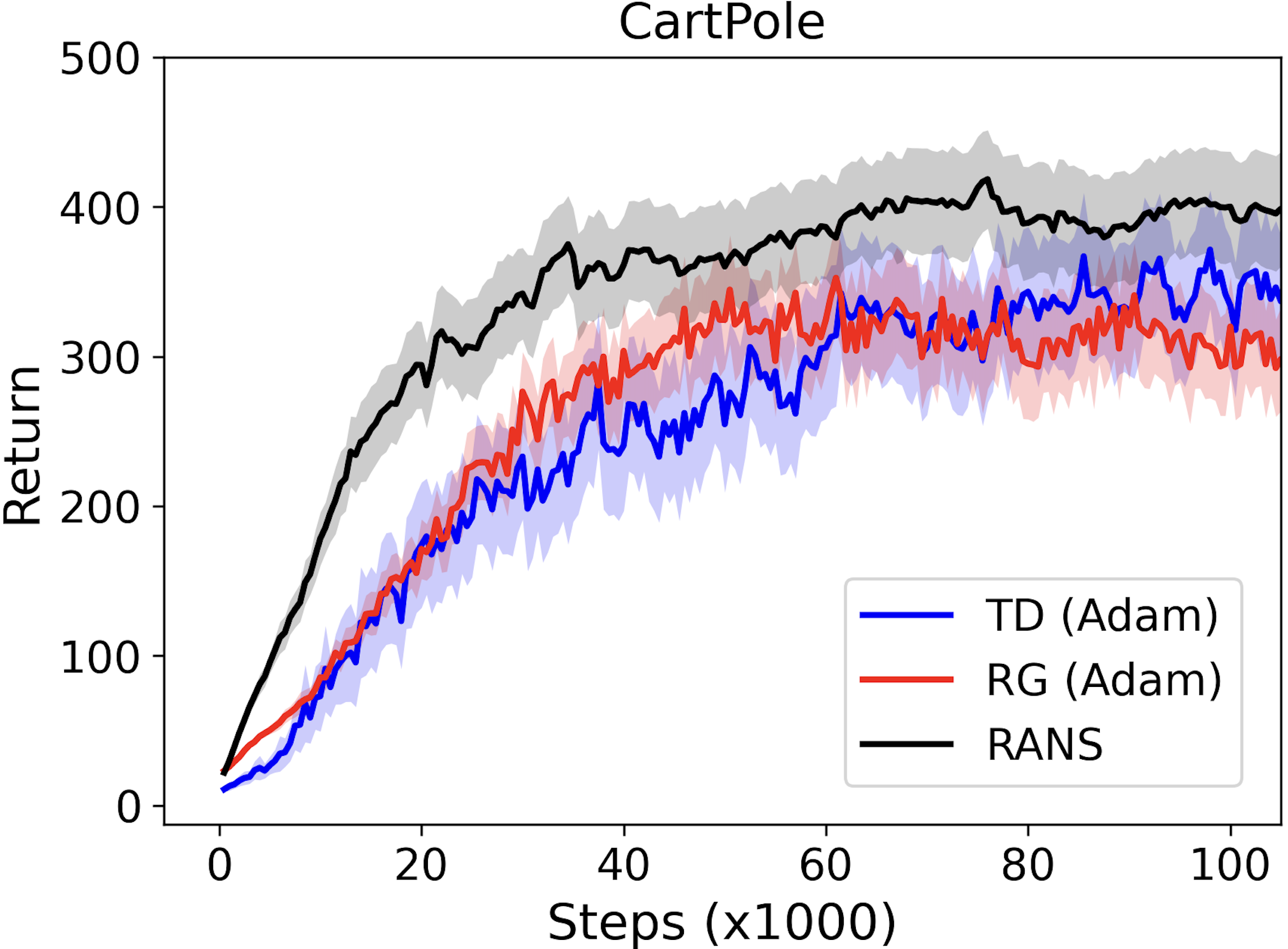}
	\end{subfigure}
	\caption{Performance of RANS, TD(0), and RG on classic control tasks. The only difference with Fig.\ref{fig:control} is that  here, each point is an average of estimate returns over 100 independent training. See Appendix~\ref{app:exp control} for details.}\label{fig:control mean}
	%\end{center}
\end{figure}

\begin{figure*}[t!]
	\centering % the figures are generated by 5_from_gcloud
	\begin{subfigure}[t]{0.4\linewidth}
		\centering
		\includegraphics[width=1\linewidth]{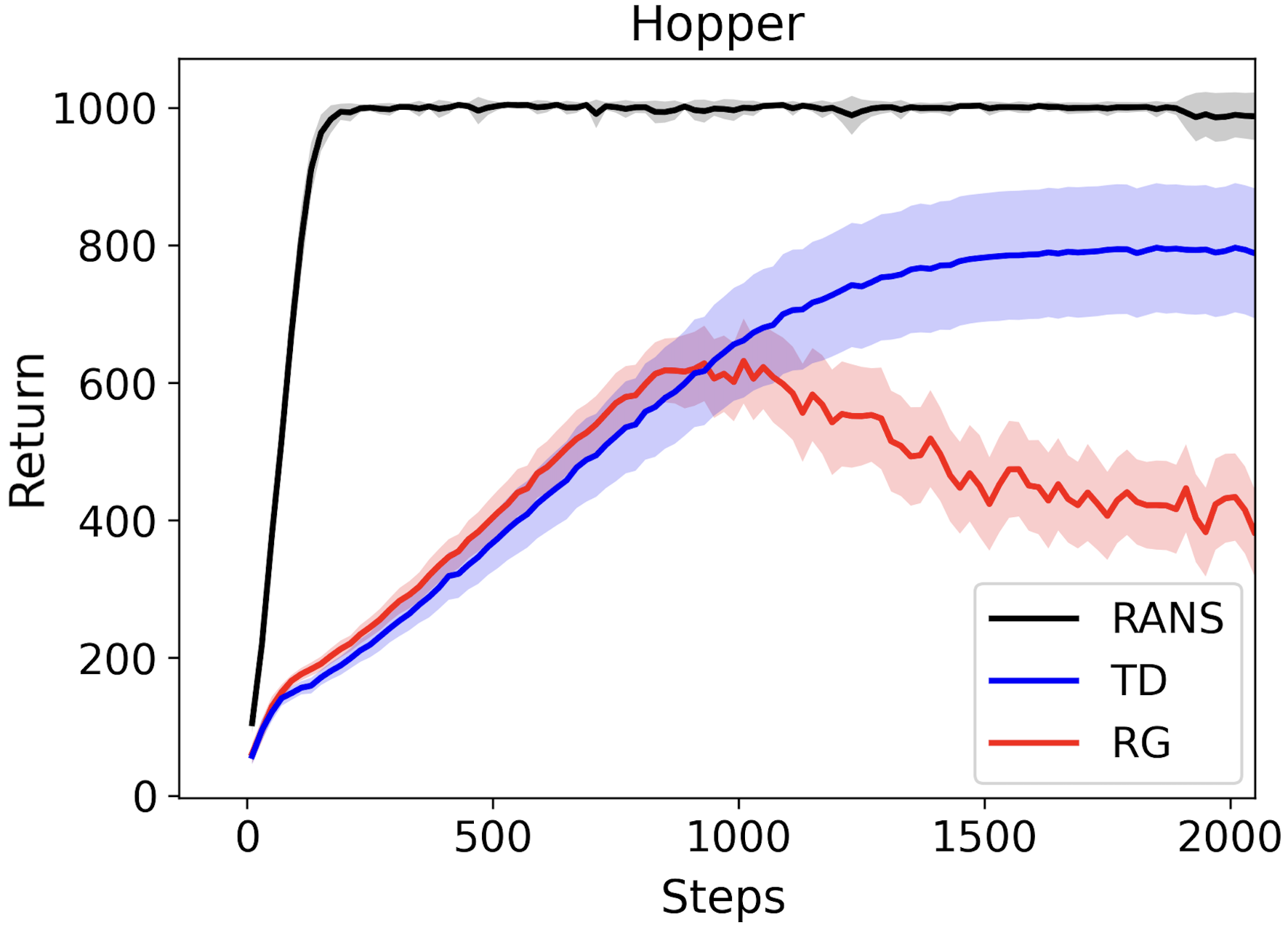}
	\end{subfigure}%
	\hspace{1.2cm}
	\begin{subfigure}[t]{.4\linewidth}
		\centering
		\includegraphics[width=1\linewidth]{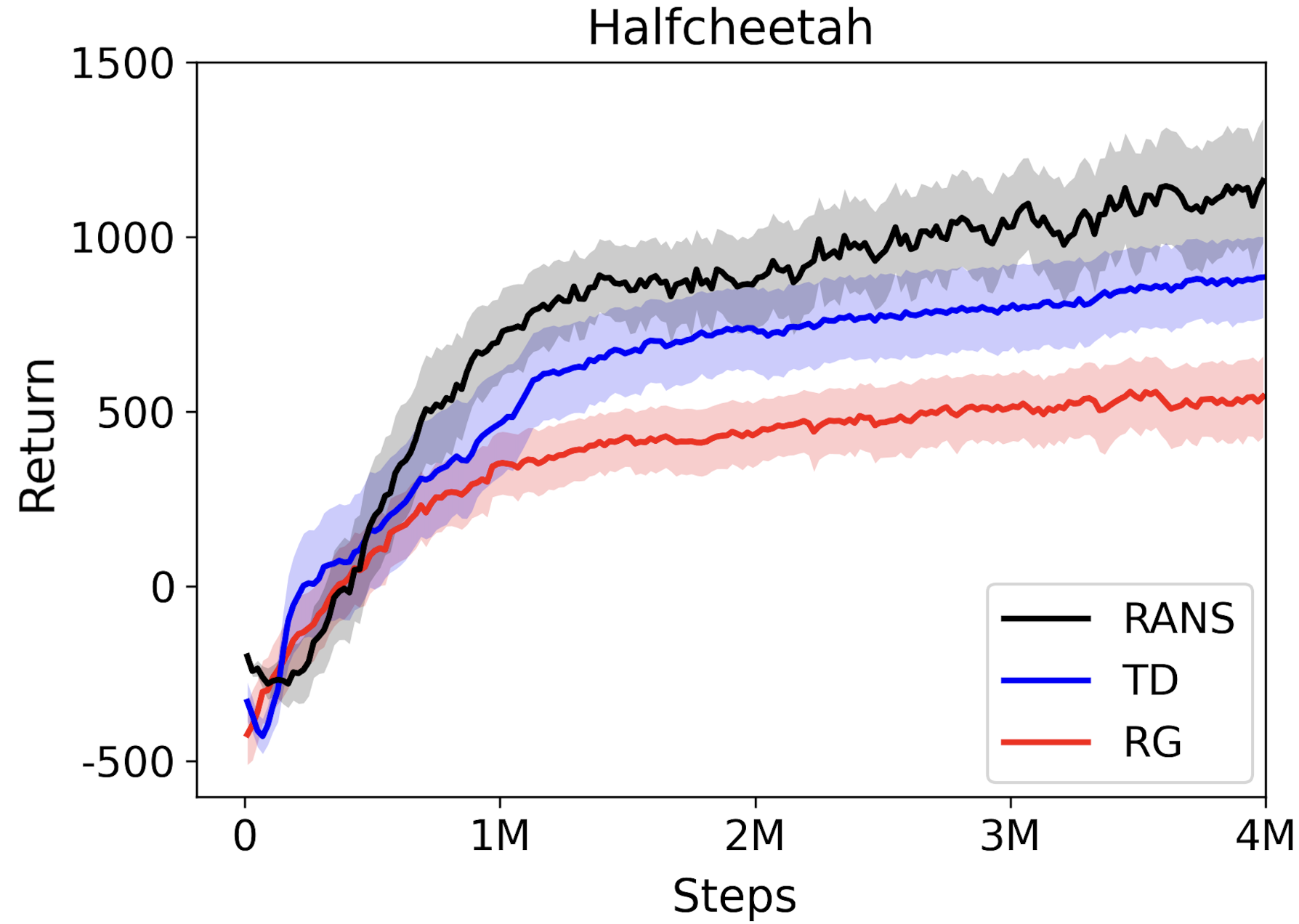}
	\end{subfigure}
	\caption{Performance of RANS, TD with target networks, and RG algorithms on simple MuJoCo environments. The only difference with Fig.\ref{fig:MuJoCo} is that  here, each point is an average of estimate returns over 100 independent training, whereas Fig.\ref{fig:MuJoCo} depicts the average of top 50 percent of estimated returns.}
	\label{fig:MuJoCo mean}
	%\end{center}
\end{figure*}

\end{document}